\documentclass[journal,10pt,onecolumn,draftclsnofoot,]{IEEEtran} 

\usepackage[pdftex]{graphicx}
\usepackage{amsmath}
\usepackage{amsfonts}
\usepackage{amssymb}
\usepackage{amsthm}
\usepackage{graphicx}
\usepackage{epstopdf}
\usepackage{algorithmic}
\usepackage{breqn}
\usepackage{epstopdf}
\usepackage{textcomp}
\usepackage{todonotes}

\newtheorem{definition}{Definition}[section]
\newtheorem{theorem}{Theorem}
\newtheorem{lemma}{Lemma}
\newtheorem*{remark}{Remark}

\newcommand\norm[1]{\left\lVert#1\right\rVert}
\newcommand\abs[1]{\left|#1\right|}
\newcommand\sech{\textrm{sech}}
\newcommand\pder[2]{\frac{\partial #1}{\partial #2}}
\newcommand\pdder[2]{\frac{\partial^2 #1}{\partial #2^2}}

\begin{document}

\title{Controlling Parent Systems Through Swarms Using Abstraction}
\author{Kyle~L.~Crandall,~\IEEEmembership{Student~Member,~IEEE,}
        and~Adam~M.~Wickenheiser~\IEEEmembership{Member,~IEEE}
\thanks{K.L. Crandall is with the Department of Mechanical and Aerospace Engineering, The George Washington University, Washington, DC, 20052 USA e-mail: crandallk@gwu.edu.}
\thanks{A.M. Wickenheiser is with The University of Delaware.}
\thanks{Manuscript received May 13, 2018; revised Jan. 18, 2019.}}

\markboth{IEEE Transactions on Control of Network Systems,~Vol.~X, No.~X, X~2019}%
{Crandall \MakeLowercase{\textit{et al.}}: Controlling Parent Systems Through Swarms Using Abstraction}
%



\maketitle
\begin{abstract}
This study considers the control of parent-child systems where a parent system is acted on by a set of controllable child systems (i.e. a swarm).  Examples of such systems include a swarm of robots pushing an object over a surface, a swarm of aerial vehicles carrying a large load, or a set of end effectors manipulating an object.  In this paper, a general approach for decoupling the swarm from the parent system through a low-dimensional abstract state space is presented.  The requirements of this approach are given along with how constraints on both systems propagate through the abstract state and impact the requirements of the controllers for both systems.  To demonstrate, several controllers with hard state constraints are designed to track a given desired angle trajectory of a tilting plane with a swarm of robots driving on top.  Both homogeneous and heterogeneous swarms of varying sizes and properties are considered to test the robustness of this architecture.  The controllers are shown to be locally asymptotically stable and are demonstrated in simulation.
\end{abstract}


%
\IEEEpeerreviewmaketitle

\section{Introduction}
%
%
%
%
\IEEEPARstart{T}{he} appeal of using multiple, cooperating agents to perform complicated tasks has been demonstrated in many applications \cite{tan2013}.  For example, when compared to a single robot, swarms can be designed to be more robust to failure, can take advantage of the inherent high degree of parallelism, and tend to be more economical with regard to reusability, maintenance, and scalability.

We consider a class of robot swarms that use multiple agents to affect a parent dynamical system.  While there are many examples of this type of system being considered, there has not been much work studying the problem as a general case of backstepping-like control; however, there do exist some frequently cited approaches such as caging.  This is where the swarm is arranged in such a way as to constrain the possible movement of the parent system.  This approach is particularly common in the case where a swarm of ground robots is manipulating a larger, passive object over a 2D surface.  Prior work has achieved caging by estimating the geometry of the object using contact with the swarm \cite{Habibi2015}, or by exploiting \textit{a priori} knowledge of the geometry \cite{Pereira2004}.  This knowledge can also be utilized to define a set of high-level behaviors for the swarm that can be achieved in a decentralized manner \cite{Fink2008}.  Robots can even be linked together to exert more force \cite{Gross2009}.  These caging approaches assume passive parent systems with at least marginally stable dynamics that are not incorporated into the design of the swarm behavior.  This precludes the evaluation of these methods' ability to handle disturbances or complex parent system dynamics.

Another similar object manipulation problem is using a swarm of multirotors to carry an object.  This can be approached using optimization, resulting in a feed forward controller based on full state feedback that takes into account a more complex model of the payload \cite{Goodarzi2015}, or by breaking the problem into a path following problem and a coordination problem \cite{Klausen2015}, resulting in a method similar to the caging methods discussed previously.  Both of these approaches isolate the dynamics of an individual member of the swarm and determine how it needs to move, either with the rest of the swarm or relative to the parent system.

This parent-child class of system we are considering includes other cases that would not traditionally be considered ``swarm'' problems, such as a walking robot, where the legs could be considered a swarm manipulating the body \cite{Hamed2017}.  An array of control surfaces could be considered a swarm acting on the wing of an aircraft \cite{Blower2012} or bridge \cite{Boberg2015}.  These problems are particularly challenging due to fluid-structure interactions with the parent systems.  In these cases, caging methods as discussed earlier would be much more difficult to apply.  While decentralized controllers for walking robots have been proposed \cite{Hamed2017}, the controllers for the arrays of flaps controlling the aerodynamics are both centralized approaches that require each element of the swarm to know the states of all the other elements.

In this study, we propose a general approach to designing controllers for this type of parent-child system.  This approach encapsulates the coupling between the dynamics of the parent and swarm systems within a lower-dimensional abstract state and the dynamics between the swarm and the abstract state within an auxiliary abstract state.  This creates a cascaded system to which we can apply a backstepping controller.  Specifically, we can design a controller for the parent system that specifies a desired abstract state of the swarm and a controller for the swarm that achieves the desired abstract state.  Due to the nature of the abstract state, this isolates the swarm and parent systems, allowing for a large degree of flexibility in the swarm size and composition at runtime.  While a centralized observer calculates the abstract states, the dimensionality of this estimation is fixed with respect to the size of the swarm.  Additionally, the response of the parent system can be used to estimate the abstract state, eliminating the need to sense and estimate it directly.

The proposed approach relies on finding an abstract state of the swarm similar to the shape parameters used in formation control \cite{Belta2004, Belta2005, Belta2007, Michael2006, Michael2008, Michael2008Aerial, Michael2010, Egerstedt2003}.  This body of work uses a low-dimensional abstraction of a swarm of robots to describe the overall geometry of the swarm.  For motion planning, this allows the distribution of the swarm to be specified rather than individual trajectories.  These studies consider only the swarm, however, and do not possess a parent system with which to interact.  Our method utilizes an abstraction depending on these interactions, facilitating the design of the swarm behavior based on the desired behavior of the parent system.

This paper formalizes the approach presented in \cite{Crandall2016} and considers constraints on the parent and swarm systems.  In section \ref{sec:method}, we formally define the parent-swarm class of problem and our proposed approach and examine the stability of the system under constraints.  In section \ref{sec:exCase} we present an example case, and we discuss simulation results in section \ref{sec:simExp}.

\section{Method}
\label{sec:method}
\begin{figure}[!t]
\centering
\includegraphics[width=6in]{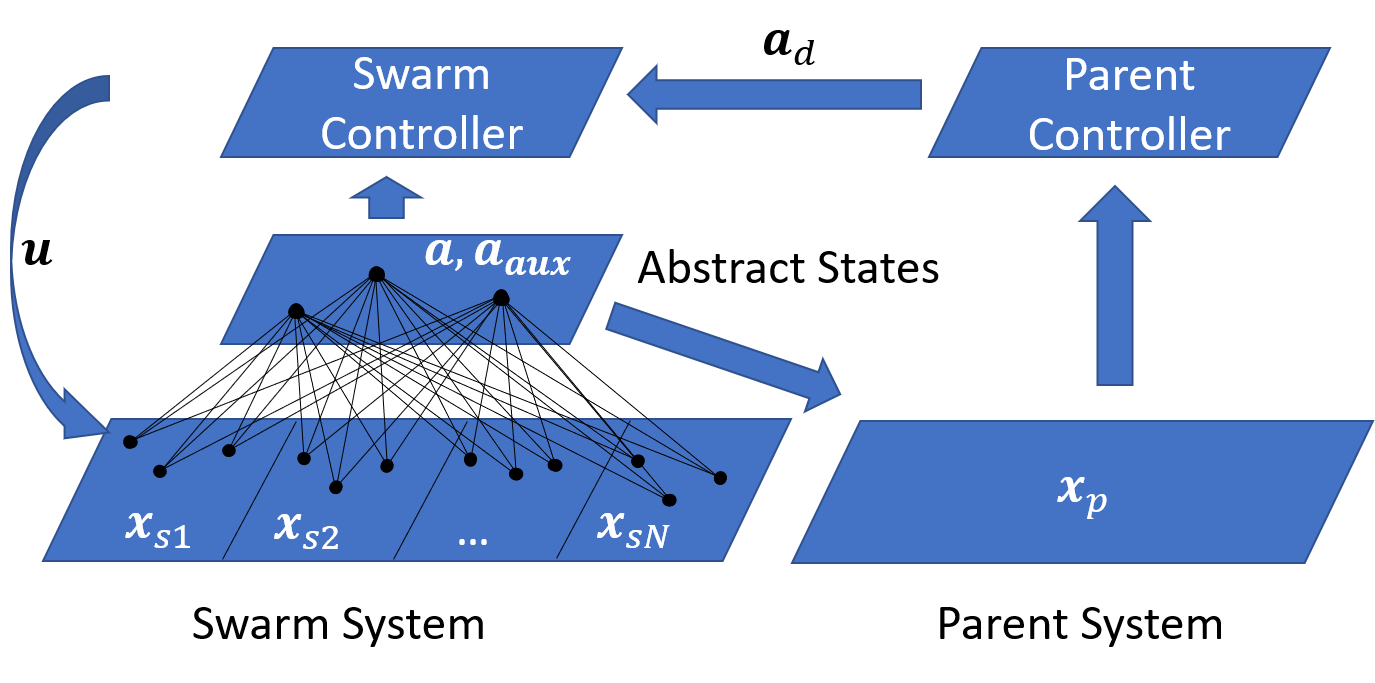}
\caption{Block Diagram depicting the the overall approach we are proposing}
\label{fig:archDiagram}
\end{figure}

We propose to control a parent-child system by finding an abstraction of the swarm's state that embodies the dynamic interactions between these two systems.  Such a system is so called because it can be reasonably segregated into two cascading systems: the parent system and a swarm of child systems.  The overall system must have the following properties:
\begin{itemize}
\item The parent system must contain the state that we seek to control.
\item The parent system's dynamics must be a function of the swarm's state.
\item The swarm's interaction with the parent system must be modular in way that it can be broken into a variable number of child systems.
\item The swarm's dynamics must be a function of the inputs to the system.
\end{itemize}
The typical dynamics for such a system are governed by the following:
\begin{align}
\dot{\mathbf{x}}_p = f_p\left(\mathbf{x}_p, \mathbf{x}_s, \mathbf{u}_s\right), &\ \textrm{where} \ f_p:X_p \times X_s \times U_s \rightarrow TX_p
\label{eq:defParentf} \\
\dot{\mathbf{x}}_s = f_s\left(\mathbf{x}_s, \mathbf{u}_s\right), &\  \textrm{where} \ f_s:X_s \times U_s \rightarrow TX_s
\label{eq:defSwarmf} \\
\dot{\mathbf{x}}_{si} = f_{si}\left(\mathbf{x}_{s}, \mathbf{u}_{si}\right), &\ \textrm{where} \  f_{si}:X_{s} \times U_{si} \rightarrow TX_{si}
\label{eq:defNodef}
\end{align}
Here, $\mathbf{x}_p$ is the state of the parent system, $\mathbf{x}_s$ is the state of the swarm system, and $\mathbf{x}_{si}$ is the state of the $i$th child system in the swarm.  We can define the spaces these states exist in as $\mathbf{x}_p \in X_p \subset \mathbb{R}^n$, $\mathbf{x}_{si} \in X_{si} \subset \mathbb{R}^{m_i}$, and $\mathbf{x}_s \in X_s = \prod_{i=1}^{N}X_{si}$ where $n$ is the dimension of the parent state, $m_i$ is the dimension of the $i$th child's state, and $N$ is the size of the swarm.  $\mathbf{u}_{s}$ is the input to the system and is a concatenation of inputs $\mathbf{u}_{si}$, which are the inputs to the $i$th child.  Thus, $\mathbf{u}_{si} \in U_{si} \subset \mathbb{R}^{p_i}$ and $\mathbf{u}_s \in U_s = \prod_{i=1}^{N}U_{si}$.  The notation $TX_p$ and $TX_s$ denote the tangent space of $X_p$ and $X_s$ respectively.

Figure \ref{fig:archDiagram} is a diagram of the overall architecture of the proposed approach.  Note that it is similar to a backstepping architecture, where the controller of the parent system specifies a desired intermediate state, and then the controller on the swarm uses that desired intermediate state to calculate how it should move to achieve it.  The innovation of our proposal is the introduction of the abstract state based on the interactions between the swarm and parent systems, allowing the dynamics of the swarm that are of little consequence to the parent system to be truncated.

\subsection{Abstraction of the Swarm}
The crux of our approach is the idea of abstracting the complex interactions of the swarm and the parent systems.  This idea is based on work by Belta and Kumar, where they propose using abstraction to control formations of robots \cite{Belta2004}.  They find an abstract state of the swarm that represents the formation's geometric properties such as position, orientation, and distribution.  Rather than specifying states that necessarily define the geometric features of the swarm, we use states that encompass the coupling dynamics between the swarm and the parent system.

\begin{definition}[Abstraction of a Swarm with Respect to a Parent System]
\label{def:swarmAbs}
An abstraction of the swarm is a mapping from the state of the swarm to an abstract state that encapsulates all the interactions between the swarm and the parent systems as described below in \eqref{eq:redefParentSysComposition}.  The dimension of this abstract state must be independent of the number of child systems in the swarm.
\end{definition}

This abstraction consists of an abstract state $\mathbf{a} \in A \subset \mathbb{R}^m$, and a mapping from the swarm state to the abstract state:

\begin{equation}
\mathbf{a} = \phi\left(\mathbf{x}_s, \mathbf{u}_s\right) \ \textrm{where} \ \phi : X_s \times U_s \rightarrow A
\label{eq:mapping}
\end{equation}
This mapping must be surjective, i.e. every abstract state must represent at least one swarm state.  We further require the dynamics of the parent system to be rewritten as
\begin{equation}
\dot{\mathbf{x}}_p = \hat{f}_p\left(\mathbf{x}_p, \mathbf{a}\right), \ \textrm{where} \ \hat{f}_p:X_p \times A \rightarrow TX_p
\label{eq:redefParentSys}
\end{equation}
and $f_p$ is a composition of $\hat{f}_p$ and $\phi$:
\begin{equation}
f_p\left(\mathbf{x}_p, \mathbf{x}_s, \mathbf{u}_s\right) = \hat{f}_p\left(\mathbf{x}_p, \phi\left(\mathbf{x}_s, \mathbf{u}_s\right)\right)
\label{eq:redefParentSysComposition}
\end{equation}
Since the dynamics of the parent state can be framed as dependent on this abstract state, the abstract state can be estimated based on the response of the parent system.  This allows for the estimation of the abstract state without a global observer, though one can still be used.

Once an abstraction of the swarm has been found, we can design a control law for the parent system in which the abstract state is considered to be the input and the swarm is considered to be independent of the parent (which follows from \eqref{eq:defNodef}).  In general, the swarm control law is a function of the parent and swarm states, i.e. $\mathbf{u}_s = \mathbf{u}_s(\mathbf{x}_p,\mathbf{x}_s)$.  We can rewrite this relationship in terms of the abstract state as follows:
\begin{equation}
\mathbf{u}_s = \mathbf{u}_s(\mathbf{a},\mathbf{a}_d,\mathbf{a}_{aux})
\label{eq:SwarmControlLaw}
\end{equation}
where $\mathbf{a}_d$ is the desired abstract state as determined by the parent controller, and $\mathbf{a}_{aux}$ is the auxiliary abstract state, which collects the remaining terms in the control law.

\begin{definition}[Auxiliary Abstraction of the Swarm]
\label{def:swarmAux}
Given an abstract state $\mathbf{a}$ from Definition \ref{def:swarmAbs}, an Auxiliary Abstract State $\mathbf{a}_{aux}$ embodies the interactions within the swarm in its control law.
\end{definition}

The auxiliary abstract state $\mathbf{a}_{aux} \in A_{aux} \subset \mathbb{R}^{m_{aux}}$ is mapped from the full swarm state $\mathbf{a}_{aux} = \phi_{aux}\left(\mathbf{x}_s\right)$, where $\phi_{aux}: X_s \rightarrow A_{aux}$.

\subsection{Constraint Considerations for Controller Design}
Any constraints on the parent and swarm systems must also be considered in the abstraction.  To do this, let $X_{sc} \subseteq X_s$ be the set of all swarm states that satisfy all constraints.  Because the abstract mapping is not necessarily injective, a given abstract state can correspond to a set of swarm states, some of which may lie outside $X_{sc}$.  Thus, we partition the set of abstract states into three disjoint sets: absolutely constrained abstract states, partially constrained abstract states, and unconstrained abstract states.

\begin{definition}[Absolutely Constrained Abstract State]
\label{def:absConstrained}
An abstract state is said to be absolutely constrained if the set it maps to in the swarm state space exists entirely within the constraints on the swarm.  The set of all absolutely constrained abstract states is the set $A_A = \left\lbrace \mathbf{a} \in A\ |\ \phi^{-1}\left[\mathbf{a}\right] \subseteq X_{sc}\right\rbrace$.
\end{definition}

\begin{definition}[Unconstrained Abstract State]
\label{def:unconstraind}
An abstract state is said to be unconstrained if the set it maps to in the swarm state violates at least one constraint on the swarm state at all points.  The set of all unconstrained abstract states is the set $A_U = \left\lbrace\mathbf{a} \in A\ |\ \phi^{-1}\left[\mathbf{a}\right] \subseteq X_s \setminus X_{sc}\right\rbrace$.
\end{definition}

\begin{definition}[Partially Constrained Abstract State]
\label{def:partConstrained}
An abstract state is said to be partially constrained if the set it maps to in the swarm state violates any of the constraints on the swarm at some points, but not all points.  The set of all partially constrained abstract states is thus $A_P = A \setminus \left( A_A \cup A_U \right)$
\end{definition}

We choose a domain $A_C$ on which we constrain the abstract state.
\begin{equation}
A_C \subseteq A_A \cup A_P
\label{eq:absConstraint}
\end{equation}
If $A_C \subseteq A_A$, the swarm controller does not have to consider the swarm constraints since all configurations of the swarm that correspond to that abstract state satisfy the constraints on the swarm.  If $A_C \cap A_P \neq \emptyset$, the swarm controller must ensure that it drives the swarm to a configuration that does not violate the constraints on the swarm.

\subsection{Stability}
Given Lyapunov functions for the parent and swarm subsystems, we can derive sufficient conditions for the stability of the closed loop system.  The closed-loop system is broken into two parts: $\dot{\mathbf{x}}_p = f_p^{\prime}\left(\mathbf{x}_p\right)$ is the closed-loop parent system, and $\dot{\mathbf{a}} = f_a^{\prime}\left(\mathbf{a}\right)$ is the closed-loop swarm system in the abstract space.  First, we consider the stability of systems without constraints.

\begin{theorem} 
\label{thm:GeneralStability}
Consider the following systems:
\begin{equation}
\dot{\mathbf{x}}_p = f_p^{\prime}\left(\mathbf{x}_p\right)
\label{eq:thm1Parent}
\end{equation}
\begin{equation}
\dot{\mathbf{a}} = f_a^{\prime}\left(\mathbf{a}\right)
\label{eq:thm1Abstract}
\end{equation}
Let $V_p\left(\mathbf{x}_p\right)$ be a Lyapunov candidate for the parent system, and let $V_a\left(\mathbf{a}\right)$ be a Lyapunov candidate for the abstract system.  Let $\mathbf{x}_p^T = \begin{bmatrix}\mathbf{z}_p^T & \mathbf{y}_p^T \end{bmatrix}$ and $\mathbf{a}^T = \begin{bmatrix}\mathbf{z}_a^T & \mathbf{y}_a^T \end{bmatrix}$ be partitioned such that
\begin{equation}
\dot{V}_p \leq \dot{V}_{p}^\prime\left(\mathbf{z}_p\right) + \dot{V}_{pc}\left(\mathbf{z}_p, \mathbf{z}_a\right)
\end{equation}
\begin{equation}
\dot{V}_a \leq \dot{V}_{a}^\prime\left(\mathbf{z}_a\right) + \dot{V}_{ac}\left(\mathbf{z}_p, \mathbf{z}_a\right)
\end{equation}
with domains $\mathbf{z}_p \in D_{zp}$, $\mathbf{z}_a \in D_{za}$, $\mathbf{y}_p \in D_{yp}$ and $\mathbf{y}_a \in D_{ya}$.  The terms $\dot{V}^{\prime}_{p}$ and $\dot{V}^{\prime}_{a}$ represent the terms of $\dot{V}_p$ and $\dot{V}_a$ that are solely dependent on the parent and abstract states, respectively.  The terms $\dot{V}_{pc}$ and $\dot{V}_{ac}$ represent the coupled terms.  Note that any of these terms can be $0$, i.e. the systems may not be coupled, or the coupling may be in one direction.

If $\dot{V}_{p}^\prime$ and $\dot{V}_{a}^\prime$ are negative definite, and
\begin{equation}
\abs{\dot{V}_{p}^\prime + \dot{V}_{a}^\prime} > \abs{\dot{V}_{pc} + \dot{V}_{ac}} \label{eq:inequality}
\end{equation}
when $\mathbf{z}_p \in D_{zp} \setminus \left\lbrace 0 \right\rbrace,\ \mathbf{z}_a \in D_{za} \setminus \left\lbrace 0 \right\rbrace$, and
\begin{equation}
\abs{\dot{V}_{pc}\left(0, 0\right) + \dot{V}_{ac}\left(0, 0\right)} = 0,
\label{eq:lyap_zero}
\end{equation}
then the origin of this system is locally stable with a region of attraction in the neighborhood of the origin such that if $\begin{bmatrix}\mathbf{z}_p^T & \mathbf{z}_a^T & \mathbf{y}_p^T & \mathbf{y}_a^T\end{bmatrix}^T  \in S \subset D_{zp} \times D_{za} \times D_{yp} \times D_{ya}$, then $\mathbf{z}_p \to 0$ and $\mathbf{z}_a \to 0$ as $t \to \infty$ and $\mathbf{y}_p$ and $\mathbf{y}_a$ remain bounded.
\end{theorem}
\begin{proof}
Let $V\left(\mathbf{z}_p, \mathbf{y}_p, \mathbf{z}_a, \mathbf{y}_a\right) = V_p\left(\mathbf{z}_p, \mathbf{y}_p\right) + V_a\left(\mathbf{z}_a, \mathbf{y}_a\right)$.  Because $V_p$ and $V_a$ are positive definite, $V$ is also positive definite.

We now take the derivative of $V$ with respect to time:
\begin{equation}
\dot{V} \leq \dot{V}_{p}^\prime\left(\mathbf{z}_p\right) + \dot{V}_{a}^\prime\left(\mathbf{z}_a\right) + \dot{V}_{pc}\left(\mathbf{z}_p, \mathbf{z}_a\right) + \dot{V}_{ac}\left(\mathbf{z}_p, \mathbf{z}_a\right) \leq -\abs{\dot{V}_{p}^\prime\left(\mathbf{z}_p\right) + \dot{V}_{a}^\prime\left(\mathbf{z}_a\right)} + \abs{\dot{V}_{pc}\left(\mathbf{z}_p, \mathbf{z}_a\right) + \dot{V}_{ac}\left(\mathbf{z}_p, \mathbf{z}_a\right)}
\end{equation}
Therefore, for $\dot{V}$ to be negative semi-definite,
\begin{equation}
\abs{\dot{V}_{p}^\prime + \dot{V}_{a}^\prime} \geq \abs{\dot{V}_{pc} + \dot{V}_{ac}}
\label{eq:Vleq}
\end{equation}

Next, we find where $\dot{V} = 0$.  This should only be true at the origin to satisfy the Lyapunov stability theorem (Theorem 4.1, \cite{Khalil2002}); therefore, we determine that, for $\mathbf{z}_p \neq 0$ and $\mathbf{z}_a \neq 0$,
\begin{equation}
\abs{\dot{V}_{p}^\prime\left(\mathbf{z}_p\right) + \dot{V}_{a}^\prime\left(\mathbf{z}_a\right)} \neq \abs{\dot{V}_{pc}\left(\mathbf{z}_p, \mathbf{z}_a\right) + \dot{V}_{ac}\left(\mathbf{z}_p, \mathbf{z}_a\right)}
\end{equation}
This combined with \eqref{eq:Vleq} gives us the inequality in \eqref{eq:inequality}.

Finally, we consider $\mathbf{z}_p = 0$ and $\mathbf{z}_a = 0$, which gives us \eqref{eq:lyap_zero}.  Given these conditions, we can say that $\dot{V}$ is negative semi-definite on the domain $D = D_p \times D_a \times \mathbb{R}^{dim\left(\mathbf{y}_p\right)} \times \mathbb{R}^{dim\left(\mathbf{y}_p\right)}$.  Thus, by Theorem 8.4 in \cite{Khalil2002}, the terms $\mathbf{z}_p$ and $\mathbf{z}_a$ are asymptotically stable with the region of attraction
\begin{equation}
S = \left\lbrace \mathbf{z} \in D\ \vert\ V\left(\mathbf{z}\right) \leq \min\left\lbrace V\left(\mathbf{b}\right) \vert \mathbf{b} \in bd\left(D\right)\right\rbrace\right\rbrace
\end{equation}
while $\mathbf{y}_p$ and $\mathbf{y}_a$ remain bounded; thus, $S$ is a closed neighborhood of the origin.  Note that $bd\left(D\right)$ is the boundary of the set $D$.
\end{proof}

\begin{remark}
Let $\phi\left(\mathbf{x}_s\right)$ be an abstraction of $\mathbf{x}_s$.  If $\phi^{-1}\left[D_{za}\right]$ is bounded, where $D_{za}$ is defined as in Theorem \ref{thm:GeneralStability}, then the swarm state is also bounded.  However, if this inverse mapping from an abstract state is unbounded, then we cannot make any inference about the boundedness of the swarm state based on the boundedness of the abstract state.
\end{remark}

We can now consider how the constraints on the system affect the region of attraction of the system.

\begin{theorem}
\label{thm:GeneralConstrainedStability}
Consider a parent-child system with constrained spaces $X_{pc}$ and $X_{sc}$ for the parent and swarm states, respectively, and let $A_C$ be a constrained abstract space that contains the origin.  Let $A_d \subseteq A_C$ be a region in which the desired abstract state is further constrained.  Let $\mathbf{w}_p\left(t\right) = \begin{bmatrix}\mathbf{z}_p^T\left(t\right) & \mathbf{y}_p^T\left(t\right)\end{bmatrix}^T$ and $\mathbf{w_a}\left(t\right) = \begin{bmatrix}\mathbf{z}_a^T\left(t\right) & \mathbf{y}_a^T\left(t\right)\end{bmatrix}^T$ be trajectories in the parent and abstract spaces that satisfy the conditions of Theorem \ref{thm:GeneralStability}. Let $D_p = D_{zp} \times D_{yp}$ and $D_a = D_{za} \times D_{ya}$.  If there exist nonempty sets
\begin{equation}
D_{pc} = \left\lbrace \mathbf{w}_p \in D_p \ \vert \ \mathbf{x}_{pd}\left(t\right) + \mathbf{b}\norm{\mathbf{w}_p} \in X_{pc}, \forall\ t \in \left[0, \infty\right), \mathbf{b} \in B^{dim\left(X_p\right)} \right\rbrace
\label{eq:constParentDomain}
\end{equation}
\begin{equation}
D_{ac} = \left\lbrace \mathbf{w}_a \in D_a \ \vert \ \mathbf{a}_d + \mathbf{b}\norm{\mathbf{w}_a} \in A_C,\ \forall\ \mathbf{a}_d \in A_d, \mathbf{b} \in B^{dim\left(A\right)} \right\rbrace
\label{eq:constSwarmDomain}
\end{equation}
where
\begin{equation}
B^n = \left\lbrace \mathbf{v} \in \mathbb{R}^n\ \vert\ \norm{\mathbf{v}} = 1 \right\rbrace
\end{equation}
and $\mathbf{x}_{pd}\left(t\right)$ is the desired trajectory of the parent state, then there exists a nonempty domain $S_c$ on which the origin of the system is asymptotically stable and $D_{pc}$ and $D_{ac}$ are forward time-invariant for all future time, where $S_c$ is given by
\begin{equation}
S_c = \left\lbrace \mathbf{z} \in S \cap D_c\ \vert\ V_p\left(\mathbf{z}\right) + V_a\left(\mathbf{z}\right) \leq  \min\left\lbrace V_p\left(\mathbf{b}\right) + V_a\left(\mathbf{b}\right) \mid b \in bd\left(S \cap D_c \right) \right\rbrace\right\rbrace
\end{equation}
and $D_c = D_{pc} \times D_{ac} \times \mathbb{R}^{dim\left(\mathbf{y}_p\right)} \times \mathbb{R}^{dim\left(\mathbf{y}_a\right)}$
\end{theorem}
\begin{proof}
\begin{figure}
\centering
\includegraphics[width=6in]{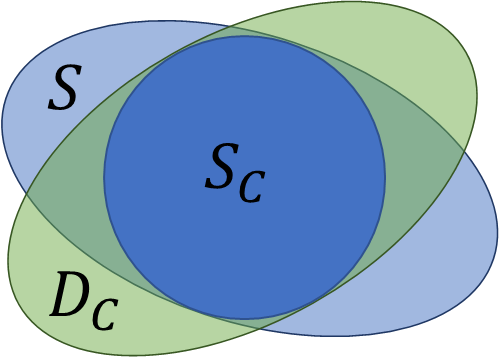}
\caption{Diagram of the sets defined in Theorems \ref{thm:GeneralStability} and \ref{thm:GeneralConstrainedStability}.  $S$ is the region of attraction of the origin of \eqref{eq:thm1Parent} and \eqref{eq:thm1Abstract}, $D_c$ is the subset of additive perturbation vectors that do not violate any constraints, and $S_c$ is the maximal set of contours of $V\left(\mathbf{z}\right)$ contained within $S \cap D_c$.}
\label{fig:setRelations}
\end{figure}
Let $\mathbf{z}_p \times \mathbf{z}_a \times \mathbf{y}_p \times \mathbf{y}_a \in S_c$.  Then $\mathbf{z}_p \in D_{pc}$, and so $\mathbf{x}_{pd}\left(t\right) + \mathbf{b} \norm{\mathbf{z}_p} \in X_p$.  By Theorem \ref{thm:GeneralStability}, $\mathbf{z}_p$ and $\mathbf{z}_a$ converge to 0 asymptotically with a region of attraction $S$. By Theorem 8.4 in \cite{Khalil2002}, if $\mathbf{z}_0 \in S_c$, then $V\left(\mathbf{z}\left(t\right)\right) \rightarrow 0$ and hence $\mathbf{z}\left(t\right) \in S_c$ as $t \rightarrow \infty$.  The relationships between these sets are illustrated in figure \ref{fig:setRelations}.

The sets $D_{pc}$ and $D_{ac}$ contain the origin since it is assumed that the desired trajectory of the parent system is feasible.  $S$ is the semilevel set of a positive definite Lyapunov function; therefore, the intersection $S \cap D_c$ must contain the origin.  The set $S_c$ is also a semilevel set of the same Lyapunov function, so it must also contain the origin.  Therefore, as long as the desired trajectory and desired abstract surface satisfy all constraints, the set $S_c$ at least contains the origin.
\end{proof}

Using these two theorems, we can build a stability proof for the whole system based on the individual Lyapunov functions for the controllers for the parent and swarm systems.  We can define a region of attraction that satisfies all constraints by mapping the constraints on the separate subsystems.  Note that if the system is not globally asymptotically stable, the estimate of the region of attraction tends to be quite conservative, which is exacerbated in the presence of constraints.  The key point to these theorems however lies in equation \eqref{eq:inequality}, which states that the stability of the overall closed-loop system is dependent on the coupling terms in the Lyapunov candidate being dominated by the non-coupling terms.

\section{Example Case}
\label{sec:exCase}
\begin{figure}[!t]
\centering
\includegraphics[width=6in]{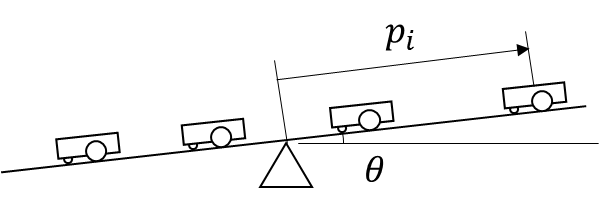}
\caption{Diagram of example system.  The generalized coordinates are shown and labeled.  $\theta$ is the angle the tilting plane makes with the ground, and $p_i$ is the distance robot $i$ is from the axis of rotation of the plane.}
\label{fig:systemDiagram}
\end{figure}

One of the key advantages of the system we propose is its modularity.  We demonstrate this by developing two different controllers for an example parent system based on techniques pulled from the literature and likewise swarm controllers for two different swarms.  We present simulation results of various combinations of these systems and controllers to demonstrate the modularity of this approach.  We show that the designed controllers can be used interchangeably, and that similar results for a given parent system controller can be achieved by any size of the two swarms considered, as well as by heterogeneous swarms consisting of members of both types of swarm.

To demonstrate our proposed controller architecture, we have chosen to apply it to a swarm of robots driving on top of a tilting plane.  Our controller drives the robots to balance the plane at a desired angle or track a desired tilt angle trajectory.  Figure \ref{fig:systemDiagram} is a diagram of the proposed system.  $\theta$ is the angle the surface of the tilting plane makes with the ground, and $p_i$ is the position of robot $i$ in the plane's coordinate system, where $p_i = 0$ at the axis of rotation of the plane.  Robot $i$ has mass $m_i$.  The robots are constrained to move in lanes to avoid collisions with each other; this is effective as the plane is only allowed to tilt about one axis.  This system was selected as it is simple enough to be easily modeled, and it demonstrates the advantages of our approach, yet it is complex enough that a control solution is not trivial.

In this system, the swarm is the group of $N$ robots driving on the plane, and the parent system is the plane itself.  To examine the interactions between the parent and swarm systems, we state the equation of motion of the parent system:
\begin{dmath}
\left(J + \displaystyle\sum_{i=1}^{N}m_i p^2_i\right)\ddot{\theta} + g\cos\left(\theta\right)\displaystyle\sum_{i=1}^{N}m_i p_i + 2\dot{\theta}\displaystyle\sum_{i=1}^{N}m_i p_i \dot{p}_i + f_f\left(\dot{\theta}\right) = 0
\label{eq:fullDynamics}
\end{dmath}
where $J$ is the plane's moment of inertia about its axis, $f_f\left(\dot{\theta}\right)$ is the Stribeck friction on the axis of rotation, and $g$ is the acceleration due to gravity.

Given these dynamics, we can identify the terms that embody the interaction between the parent system and the swarm.  We choose the abstract state given in \eqref{eq:abstractState} and rewrite the dynamics of the parent system in terms of this state in \eqref{eq:newDynamics}:

\begin{equation}
\mathbf{a} = \begin{bmatrix} \tau_s \\ J_s \\ \dot{J}_s\end{bmatrix} = \begin{bmatrix}g \displaystyle\sum_{i=1}^N m_i p_i & \displaystyle\sum_{i=1}^N m_i p_i^2 & \displaystyle\sum_{i=1}^N2m_ip_iv_i\end{bmatrix}^T
\label{eq:abstractState}
\end{equation}

\begin{equation}
\dot{\mathbf{x}}_p = \begin{bmatrix} \dot{\theta} \\ \ddot{\theta} \end{bmatrix} = \begin{bmatrix} \dot{\theta} \\ \frac{-\cos\left(\theta\right)\tau_s - \dot{\theta}\dot{J}_s - f_f\left(\dot{\theta}\right)}{J + J_s} \end{bmatrix}
\label{eq:newDynamics}
\end{equation}
In \eqref{eq:newDynamics}, $\tau_s$ and $J_s$ represent the torque exerted on the plane by the weight of the swarm and the moment of inertia of the swarm, respectively.  This abstraction satisfies Definition \ref{def:swarmAbs} for a valid abstraction: the dimension of the abstract state is not dependent on the size of the swarm, and the abstract state encapsulates the interactions between the parent and swarm systems according to \eqref{eq:redefParentSysComposition}.  Note that, while the abstract state in \eqref{eq:abstractState} does contain both $J_s$ and $\dot{J}_s$, we can use an abstract state that only contains $\tau_s$ and $J_s$ and show that if $J_s$ converges to a desired value, then so too does $\dot{J}_s$.  This simplifies the by eliminating the coupling between $J_s$ and $\dot{J}_s$, as well as making the abstract state solely dependant on the swarm state for all three swarms being considered next.

We now define the state equations for the swarm.  We consider three different swarm compositions: the first consists of robots that are one-dimensional single integrators controlled by velocity:
\begin{equation}
\mathbf{x}_s = \begin{bmatrix}p_1 & \cdots & p_N \end{bmatrix}^T
\end{equation}
\begin{equation}
\dot{\mathbf{x}}_s = \mathbf{u}_s
\label{eq:swarmModel}
\end{equation}

The second type of swarm consists of robots that are single-degree-of-freedom double integrators with passive linear damping  controlled with a force input.  This system has the same abstract state, but the state representation of the swarm dynamics is
\begin{equation}
\dot{\mathbf{x}}_s = \begin{bmatrix}\dot{\mathbf{p}}\\ \ddot{\mathbf{p}}\end{bmatrix} = \begin{bmatrix}0_{n\times n} & I \\ 0_{n\times n} & -M^{-1}C\end{bmatrix}\begin{bmatrix} \mathbf{p} \\ \dot{\mathbf{p}}\end{bmatrix} + \begin{bmatrix} 0_{n\times n} \\ M^{-1}\end{bmatrix}\mathbf{u}_s \label{eq:EOM_DI}
\end{equation}
where $M$ and $C$ are diagonal matrices of the mass and linear damping coefficients respectively for each robot, and $\mathbf{p}$ is a vector of the positions of each robot in the swarm.

The third swarm type is a non-homogeneous swarm consisting of members from both of the previous two types.

\subsection{Constraint Analysis}
Now that we have defined the abstraction of the system, we can consider the effects of constraints on both the parent and swarm system.  First, there is a limit on how far the plane can tilt before the robots lose their traction and slip off the plane, resulting in a constrained parent state space $X_p$ of the form

\begin{equation}
X_p = \{\mathbf{x}_p \in \mathbb{R}^2 \vert \abs{\theta} \leq \theta_{max} \}
\label{eq:parentConstraint}
\end{equation}
where $\theta_{max}$ is the maximum tilt angle.

Concerning the child systems, we consider constraints on their position designed to prevent the robots from driving off the edge of the plane, giving

\begin{equation}
X_{si} = \left\{\mathbf{x}_{si} \in \mathbb{R}\ \middle\vert\ \abs{p_i} \leq \frac{L}{2} \right\}
\label{eq:swarmStateConstraint}
\end{equation}

Furthermore, we need to describe the absolutely and partially constrained regions.  We can see that $\phi^{-1}\left(\tau_s,J_s\right)$ is a hyperplane when $\tau_s$ is held constant, and   it is a hyperellipsoid when $J_s$ is held constant.  The surface that a given abstract state maps to is the intersection of these two surfaces.  We can then check the resulting intersection against the constraints given in \eqref{eq:swarmStateConstraint}.  Figure \ref{fig:mappingEx} illustrates these inverse mappings for a 3-robot swarm.  Figure \ref{fig:posMapping} shows a Monte Carlo approximation of the abstract spaces and the regions within.

\begin{figure}
\centering
\includegraphics[width=6in]{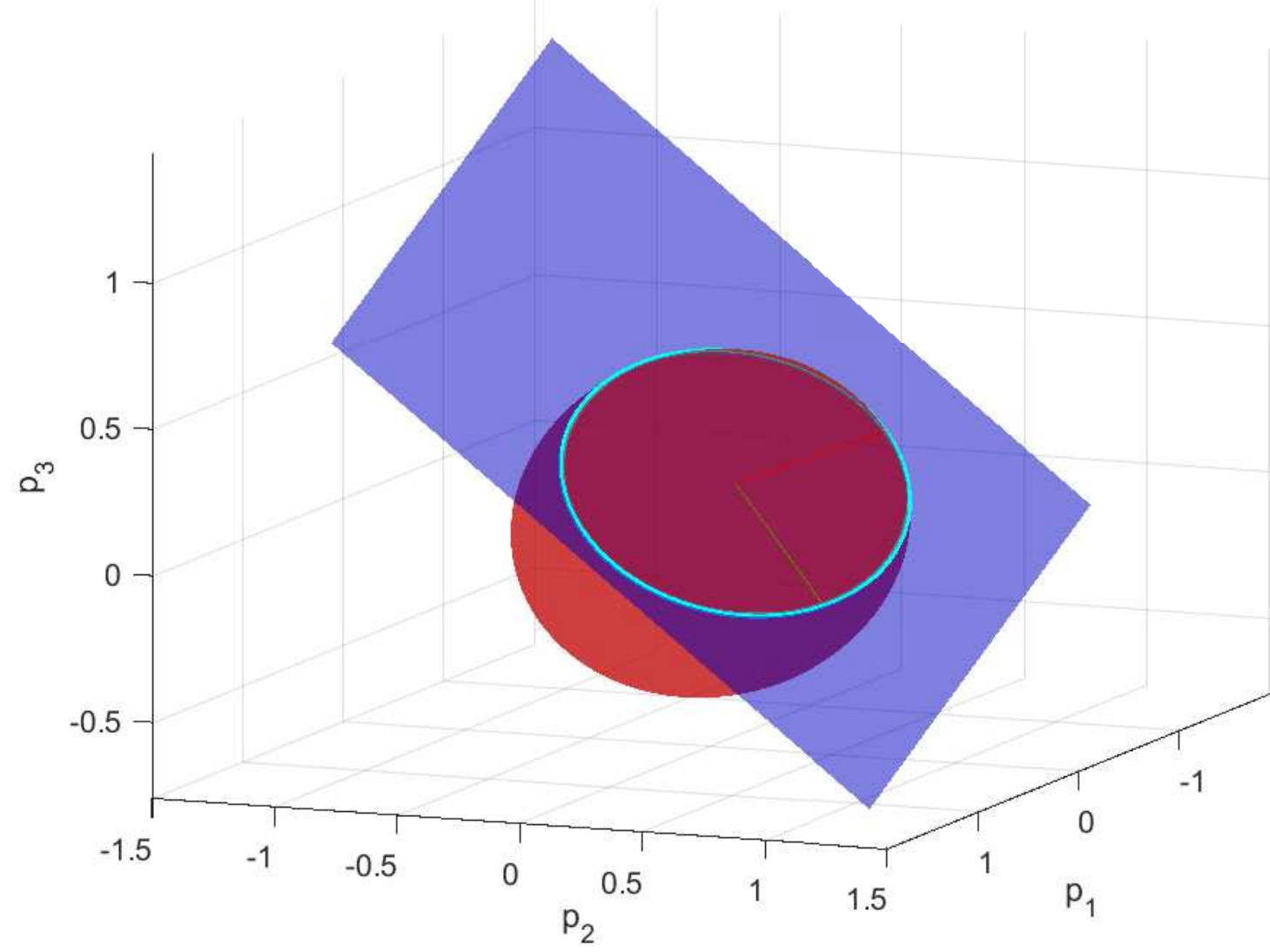}
\caption{Visualization of inverse mapping of an abstract state.  The  plane is all swarm states such that $\tau_s = 10N \cdot m$ and the red ellipsoid is all swarm such states such that $J_s = 1kg \cdot m^2$ for a three member swarm where $m_1 = 2kg$, $m_2 = 3kg$, and $m_3 = 3kg$.  The cyan ellipse on the plane is the pre-image of $\mathbf{a}=\left[ 10\ 1 \right]^T$.}
\label{fig:mappingEx}
\end{figure}

\subsection{Parent Controller}
Given the constraints found in the previous section, we can now design a controller for the parent system.  This controller must have a region of attraction contained within  $X_{pc}$.  It must also control the system while ensuring that the desired abstract state satisfies all of its constraints.  Trying to control a single degree of freedom, $\theta$, using both $\tau_s$ and $J_s$ independently would result in an over-actuated system.  To remedy this, we specify a manifold that couples $\tau_s$ and $J_s$.  We shape this manifold to exist within the interior of the absolutely constrained region of the abstract space.  The distance between the manifold and the boundary of the absolutely constrained space is exploited by Theorem \ref{thm:GeneralConstrainedStability} to show that the system is both stable and satisfies given constraints.

For the parent system, we apply two different controllers: a PD controller and an Adaptive Robust Integral on the Sign of the Error (ARISE) controller based on the design proposed by Xian et al. \cite{Xian2004} with a feed-forward adaptive term as proposed by Patre et al. \cite{Patre2006}\cite{Patre2008}. These controllers specify a desired torque $\tau_{sd}$, which is treated as the input to the parent system.  The desired swarm moment of inertia $J_{sd}$ can be determined by the position of $\tau_{sd}$ on the constraint manifold and therefore be treated as part of the dynamics of the parent system.

The PD controller regulates the plane to a zero angle with zero velocity.  Before designing the PD controller, we set $J_{sd} = J_{sd,0}$, a constant, and linearize the damping as $c$.  We then linearize the dynamics about $\mathbf{x}_p = 0$ and $\tau_{sd} = 0$, to arrive at
\begin{equation}
\dot{\mathbf{x}}_p = \begin{bmatrix}0 & 1 \\ 0 & \frac{c}{J + J_{sd,0}}\end{bmatrix}\mathbf{x}_p + \begin{bmatrix} 0 \\ \frac{-1}{J + J_{sd,0}}\end{bmatrix}\tau_{sd}
\end{equation}

With this linearization, we can now use a linear control law:
\begin{equation}
\tau_{sd} = K_{pd}\mathbf{x}_p
\label{eq:ControlLawPD}
\end{equation}
where $K_{pd} \in \mathbb{R}^{1\times2}$ is the matrix of control gains.  This matrix is set so that this system is a Linear Quadratic Regulator (LQR) that minimizes the following cost function:
\begin{equation}
\displaystyle\int_{t_0}^{\infty}\mathbf{x}_p^TQ\mathbf{x}_p\left(\tau\right) + R\tau_{sd}^2\left(\tau\right)d\tau
\end{equation}
where $Q$ is a symmetric, positive definite matrix, $R$ is a positive scalar, and $t_0$ is the initial time.

We also consider the ARISE controller, which can track trajectories bounded up to their fourth derivative in the presence of disturbances bounded up to their first derivative.  We develop this controller by first rewriting the equation of motion of the parent system as
\begin{equation}
\left(J + J_s\right)\ddot{\theta} + \dot{J}_s  \dot{\theta} + f_f\left(\dot{\theta}\right) + \tau_d = -\cos\left(\theta\right)\tau_s
\end{equation}
using the friction model
\begin{equation}
f_f\left(\dot{\theta}\right) = \gamma_1 \left(\tanh\left(\gamma_2 \dot{\theta} \right) - \tanh\left(\gamma_3 \dot{\theta} \right)\right) + \gamma_4 \tanh\left(\gamma_5 \dot{\theta} \right) + \gamma_6 \dot{\theta}
\label{eq:frictionModel}
\end{equation}
where $\gamma_i$ are the parameters of the friction model, and $\tau_d$ models the Coriolis term as a bounded disturbance.  The values of $\gamma_i$ are all positive, where $\gamma_2 > \gamma_3$ and $\gamma_4 \leq \gamma_1$.

The parameters of the friction model may be estimated using system identification techniques; however, we use an adaptation law to estimate these terms.  The uncertainties in the friction model and moment of inertia can be combined into a single term represented as a linear relationship between a vector of adaptive constants, $\lambda$, and a matrix that is a function of the desired state and its derivatives, $Y_d$:

\begin{align}
Y_d\lambda &= \begin{bmatrix}\ddot{\theta}_d \\ \tanh\left(\gamma_2 \dot{\theta}_d \right) - \tanh\left(\gamma_3 \dot{\theta}_d \right) \\ \tanh\left(\gamma_5 \dot{\theta}_d \right) \\ \dot{\theta}_d \end{bmatrix}^T \begin{bmatrix} J \\ \gamma_1 \\ \gamma_4 \\ \gamma_6 \end{bmatrix} \label{eq:adaptiveTerm} \\
&= J\ddot{\theta}_d + f_f\left(\dot{\theta}_d\right)
\end{align}
Note that $\gamma_1$, $\gamma_4$, and $\gamma_6$ are part of the adaptive state because $f_f$ is affine in them. The values of $\gamma_2$, $\gamma_3$, and $\gamma_5$ must be estimated; however, the ARISE controller has been proven to be robust to errors in these values \cite{Patre2006}

Using the error terms developed in \cite{Xian2004}\cite{Patre2006}\cite{Patre2008} given by $e_1 = \theta_d - \theta$ and $e_2 = \dot{e}_1 + \alpha_1e_1$, we apply the control law given by
\begin{align}
\tau_{sd}\left(t\right)
=&-\sec\left(\theta\left(t\right)\right)\left(Y_d\hat{\lambda} + \left(k_s + 1\right)e_2\left(t\right) - \left(k_s + 1\right)e_2\left(t_0\right) + \mu_1\left(t\right)\right) \label{eq:controlLaw} \\
\dot{\mu}_1\left(t\right) &= \left(k_s + 1\right)\alpha_2 e_2\left(t\right) + \beta sgn\left(e_2\left(t\right)\right) \label{eq:controlLaw2}\\
\dot{\hat{\lambda}}\left(t\right) &= \hat{\lambda}\left(t_0\right) + \left.\Gamma\dot{Y}_d^T\left(\tau\right) e_2\left(\tau\right)\right|_{\tau=t_0}^{\tau=t} - \Gamma\mu_2\left(t\right)\label{eq:controlLaw3} \\
\dot{\mu}_2\left(t\right) &= \ddot{Y}_d^T\left(t\right) e_2\left(t\right) - \alpha_2\dot{Y}_d^T e_2\left(t\right) \label{eq:controlLaw4}
\end{align}
which utilizes a measurement of $e_1$ and $e_2$ as well as an estimate of the unknown constants, $\hat{\lambda}$.  The gains $\alpha_1$, $\alpha_2$, $k_s$, and $\beta$ as well as the adaptive gain $\Gamma$ must be sufficiently large to guarantee stability; however, these gains should not be so large that the system violates the bounds given in \eqref{eq:tauLimit}.

Note the secant term added to the controller in \eqref{eq:controlLaw} to cancel out the cosine term in the dynamics of the parent system.  This reduces the usable range of the controller to $\pm\frac{\pi}{2}$.  However, the robots begin to slip off the plane at some $\abs{\theta} < \frac{\pi}{2}$; thus, this domain is larger than the constraint given in \eqref{eq:parentConstraint}.

\subsection{Swarm Controller}
The previous section derived a desired $\tau_{sd}$ given by the two parent control laws, the PD controller \eqref{eq:ControlLawPD} and the ARISE controller \eqref{eq:controlLaw}-\eqref{eq:controlLaw4}. This, along with the manifold used to calculate $J_{sd}$ \eqref{eq:JLimit} gives us a desired abstract state.  This section presents controllers for achieving these abstract states for the three different swarms.

\subsubsection{Single Integrator Swarm}
For the swarm described by \eqref{eq:swarmModel}, we use the input linearized controller simmilar to the one proposed by Belta and Kumar \cite{Belta2004}.  This control law is based on the derivative of the abstract state that is dependent on the errors in the abstract state given in \eqref{eq:abstractErrorDef}.  We first define the desired dynamics of the abstract state, where $K$ is a diagonal, positive definite matrix and $\mathbf{e}_a$ is the error in the abstract state:
\begin{align}
\dot{\mathbf{a}} &= K\mathbf{e}_a + \dot{\mathbf{a}}_d\label{eq:swarmControlLaw} \\
\mathbf{e}_a &= \mathbf{a}_d - \mathbf{a} = \begin{bmatrix} e_\tau \\ e_J \end{bmatrix} \label{eq:abstractErrorDef}
\end{align}
where $\mathbf{a}_d$ is the desired abstract state of the swarm as specified by the parent controller.

If we take the time derivative of the abstract mapping \eqref{eq:abstractState} we get
\begin{equation}
\dot{\mathbf{a}} = \Phi_{xs}\dot{\mathbf{x}}_s
\label{eq:mapping_d}
\end{equation}
where $\Phi_{xs}$ is the Jacobian of the mapping function $\phi$ evaluated at the current $\mathbf{x}_s$.  We combine \eqref{eq:mapping_d} with \eqref{eq:swarmControlLaw} and the system model given in \eqref{eq:swarmModel} to derive the control law given by
\begin{equation}
\mathbf{u}_s = \Phi_{xs}^\dagger \left(K\mathbf{e}_a + \dot{\mathbf{a}}_d\right)
\label{eq:swarmControlLaw2}
\end{equation}

The pseudo-inverse of the Jacobian can be written in terms of a set of parameters, $S_0, \dots, S_3$, that are constant across the whole swarm:
\begin{equation}
\Phi_{xs}^\dagger = \frac{1}{S_3} \begin{bmatrix}
m_1S_2 - m_1p_1S_1 & \frac{1}{2}\left(m_1p_1S_0 - m_1S_1\right)\\
\vdots & \vdots\\
m_NS_2 - m_Np_NS_1 & \frac{1}{2}\left(m_Np_NS_0 - m_NS_1\right)
\end{bmatrix} \label{eq:pinv}
\end{equation}
where
\begin{equation}
\begin{split}
S_0 = \displaystyle\sum_{i=1}^N m_i^2,\ S_1 = \displaystyle\sum_{i=1}^N m_i^2p_i,\ S_2 = \displaystyle\sum_{i=1}^N m_i^2p_i^2,\\
\text{and } S_3 = \displaystyle\sum_{i=1}^N \displaystyle\sum_{j=i+1}^N m_i^2 m_j^2 \left(p_i - p_j\right)^2
\end{split}
\end{equation}
We can combine these sums into an auxiliary abstract state $\mathbf{a}_{aux} = \begin{bmatrix}
S_0 & S_1 & S_2 & S_3 \end{bmatrix}^T$. Like the abstract state, this vector's size is independent of the size of the swarm.  Rather than describing how the swarm interacts with the parent system, the auxiliary abstract state describes how the swarm interacts with itself, allowing an individual swarm member to determine how it needs to move within the swarm to achieve the desired abstract state using the following control law:
\begin{equation}
u_{si} = \frac{1}{S_3}\begin{bmatrix} m_iS_2 - m_ip_iS_1 \\ \frac{1}{2}\left(m_ip_iS_0 - m_iS_1\right)\end{bmatrix}^T \left(K\mathbf{e}_a + \dot{\mathbf{a}}_d\right)
\label{eq:nodeControlLaw}
\end{equation}

Note the singularity when $S_3=0$, which occurs when the entire swarm is concentrated at the same position on the plane.  This singularity can be avoided by choosing a manifold on which the desired abstract state exists that never intersects this point.  Assuming a manifold exists in the interior of the constrained space $A_A \cup A_P$, another valid manifold that does not pass through the singularity can be found by translating the first one.

Since the system dynamics \eqref{eq:fullDynamics} and the parent control laws \eqref{eq:ControlLawPD} and \eqref{eq:controlLaw} are continuous and differentiable, the terms $J_{sd}$ from \eqref{eq:JLimit} and $J_s$ from \eqref{eq:abstractState} are also continuous and differentiable, and thus, necessarily uniformly continuous.  Hence, the error term $e_J$ is uniformly continuous. In addition, $\dot{J}_{sd}$ is also differentiable.  The control laws for the swarm, \eqref{eq:nodeControlLaw} and \eqref{eq:nodeControlLaw_double}, are both continuous and differentiable, thus $\dot{J}_s$ is also differentiable.  Hence, the error term $\dot{e}_J$ must be differentiable.  We can now use Barbalat's Lemma to show that if $e_J \rightarrow 0$ as $t \rightarrow \infty$, then $\dot{e}_J \rightarrow 0$ as $t \rightarrow \infty$ \cite{Slotine1991}.

We demonstrate the stability of this controller with both parent controllers in the following theorems, which are applications of Theorems \ref{thm:GeneralStability} and \ref{thm:GeneralConstrainedStability}.

\begin{theorem}
\label{thm:LQR_SI_sys}
Given the error vector $\mathbf{z} = \begin{bmatrix}\theta & \dot{\theta} & e_\tau & e_J \end{bmatrix}^T$, the PD control law \eqref{eq:ControlLawPD} on the parent system \eqref{eq:newDynamics} and the single integrator control law \eqref{eq:nodeControlLaw} on the swarm \eqref{eq:swarmModel} asymptotically stabilizes the origin of this system if $k_1>0$, $k_2>\frac{\dot{J}_{max}}{2\cos\left(\theta_{max}\right)}$, $K>0$, and $\mathbf{z}_0 \in S$,
where $k_1$ and $k_2$ are the elements of the matrix $K_{pd}$ in the parent controller, $K$ is the controller gain for the swarm system, $\mathbf{z}_0$ is the initial condition of the system, and $S$ is an estimate of the region of attraction
\begin{equation}
S = \left\lbrace\mathbf{z} \in \mathbb{R}^4\ \vert\ \norm{\mathbf{z}}^2 \leq \frac{\theta_{max}^2}{2\eta}\right\rbrace
\end{equation}
where $\eta = max\left\lbrace 1, \frac{J + J_{sd,max}}{2}\right\rbrace$.  The value $\theta_{max}$ is a limit on the angle $\theta$ such that $\abs{\theta} \leq \theta_{max} < \frac{\pi}{2}$.
\end{theorem}
\begin{proof}
See Appendix \ref{apx:ProofLQR_SI}.
\end{proof}

\begin{theorem}
\label{thm:ARISE_SI_sys}
Given the error vector $\mathbf{z} = \begin{bmatrix}e_1 & e_2 & r & e_{\tau} & e_J\end{bmatrix}^T$ where $r = \dot{e}_2 + \alpha_2e_2$, the ARISE control law \eqref{eq:controlLaw}, \eqref{eq:controlLaw2}, \eqref{eq:controlLaw3}, and \eqref{eq:controlLaw4} on the parent system \eqref{eq:newDynamics} and the single integrator control law \eqref{eq:nodeControlLaw} on the swarm \eqref{eq:swarmModel} asymptotically stabilize the origin of this system if $\alpha_1 > \frac{1}{2}$, $\alpha_2 > 1$, $\beta > \zeta_{N_d} + \frac{1}{\alpha_2}\zeta_{\dot{N}_{d}}$, $k_s > \frac{c_{max}}{\eta_3}$, $K > 0$, $\Gamma > 0 $, and $\Gamma^T = \Gamma$, where $\tilde{N} \leq c_{max}\norm{\mathbf{z}}$, and 
\begin{equation}
\eta_3 = min\left\lbrace 2\alpha_1 - 1, \alpha_2 - 1, 1\right\rbrace
\end{equation}

We define $\abs{N_d} \leq \zeta_{N_d}$ and $\abs{\dot{N}_d} \leq \zeta_{\dot{N}_d}$ where
\begin{equation}
N_d = (J+J_s)\dddot{\theta_d} + \dot{J}_s\ddot{\theta}_d +\dot{f}_f\left(\dot{\theta}_d\right) - \dot{Y}_d\lambda
\end{equation}
The desired trajectory satisfies $\theta_d, \dot{\theta}_d, \ddot{\theta}_d, \dddot{\theta}_d \in \mathcal{L}_{\infty}$ and
The origin of this system has a region of attraction
\begin{equation}
S = \left\lbrace \mathbf{z} \in B,\ \mathbf{y}_p \in \mathbb{R}^2\ \vert\  \norm{\begin{bmatrix} \mathbf{z}^T & \mathbf{y}_p^T \end{bmatrix}^T}^2 \leq \frac{\eta_1}{\eta_2}\rho^2 \right\rbrace
\end{equation}
where $\rho = \rho_E^{-1}\left(2\eta_3 - \frac{c_{max}}{2k_s} + 2\lambda_{min}\right)$, $\lambda_{min}$ is the smallest eigenvalue of $K$, and
\begin{equation}
\eta_1 = \frac{1}{2}\min\left\lbrace 1, J + J_{sd,0}, \lambda_{min}\left(\Gamma^{-1}\right)\right\rbrace
\end{equation}
\begin{equation}
\eta_2 = \frac{1}{2}\max\left\lbrace 2, J + J_{sd,max}, \lambda_{max}\left(\Gamma^{-1}\right)\right\rbrace
\end{equation}
\end{theorem}
\begin{proof}
See Appendix \ref{apx:ProofARISE_SI}.
\end{proof}

\subsubsection{Double Integrator Swarm}
The swarm whose dynamics are given by \eqref{eq:EOM_DI} constitute a second-order system whose control input only directly effects the acceleration.  For this reason, we specify second-order dynamics for the abstract state:
\begin{equation}
\ddot{\mathbf{a}}=K_p\mathbf{e}_a + K_d\dot{\mathbf{e}}_a + \ddot{\mathbf{a}}_d
\label{eq:swarmControlLaw_double}
\end{equation}
where $K_p$ and $K_d$ are diagonal, positive definite matrices.

We choose a control law
\begin{equation}
\mathbf{u}_s = M\Phi^\dagger\left(K_p\mathbf{e}_a + \left(K_d - C_a\right)\dot{\mathbf{e}}_a - \dot{\Phi}_{xs}\dot{\mathbf{p}} + \ddot{\mathbf{a}}_d\right) + \left(C + k_{sd}I\right)\Phi_{xs}^\dagger\dot{\mathbf{a}}_d - k_{sd}\dot{\mathbf{p}}
\label{eq:swarmControlLaw_double2}
\end{equation}
where $k_{sd}$ is a positive scalar that adds damping to the individual members of the swarm, $I$ is the identity matrix, and
\begin{equation}
\begin{split}
C_a &= \begin{bmatrix}
C_{a11} & C_{a12} \\
C_{a21} & C_{a22}
\end{bmatrix} \\
C_{a11} &= \sum_{i=1}^{N}m_i\left(k_{sd} + c_i\right)\left(S_2 - p_iS_1\right) \\
C_{a12} &= \frac{1}{2}\sum_{i=1}^{N}m_i\left(k_{sd} + c_i\right)\left(p_iS_0 - S_1\right) \\
C_{a21} &= 2\sum_{i=1}^{N}m_ip_i\left(k_{sd} + c_i\right)\left(S_2 - p_iS_1\right) \\
C_{a22} &= \sum_{i=1}^{N}m_ip_i\left(k_{sd} + c_i\right)\left(p_iS_0 - S_1\right)
\end{split}
\end{equation}
This feed-forward controller enforces the desired second-order dynamics in \eqref{eq:swarmControlLaw_double} on the abstract state, adds damping to each child system, and compensates the nonlinear effects of the damping in each child system.  We build an auxiliary abstract state like we did with the single integrator system such that $\mathbf{a}_{aux}$ contains all the $S_i$ and $C_{ajk}$ terms for $i=0\dots3$, $j=1,2$, and $k=1,2$.  We can now use this auxiliary abstract state along with the abstract state, desired abstract state, and the state of a member of the swarm to create a control law for that member:
\begin{equation}
u_{si} = \frac{1}{S_3}\begin{bmatrix} m_iS_2 - m_ip_iS_1 \\ \frac{1}{2}\left(m_ip_iS_0 - m_iS_1\right)\end{bmatrix}^T \left(m_i\left(K_p\mathbf{e}_a + \left(K_d - C_a\right)\dot{\mathbf{e}}_a - \dot{\Phi}_{xs}\dot{p}_i + \ddot{\mathbf{a}}_d\right) + \left(c_i+k_{sdi}\right)\dot{\mathbf{a}}_d\right) - k_{sd}\dot{p}_i
\label{eq:nodeControlLaw_double}
\end{equation}
 Note that this control law, like the one given in \eqref{eq:nodeControlLaw}, does not require knowledge of the other child systems' states, only the abstract states.  This system also has the same singularity point when $S_3 = 0$, which can be dealt with similarly.
We demonstrate the stability of this controller with both parent controllers in the following theorems.
\begin{theorem}
\label{thm:LQR_DI_sys}
Given the error vector $\mathbf{z} = \begin{bmatrix}\theta & \dot{\theta} & e_\tau & e_J \end{bmatrix}^T$, the PD control law \eqref{eq:ControlLawPD} on the parent system \eqref{eq:newDynamics}, and the double integrator control law \eqref{eq:nodeControlLaw_double} on the swarm \eqref{eq:EOM_DI} asymptotically stabilizes the origin of this system if the gains for the parent controller satisfy the same conditions as Theorem \ref{thm:LQR_SI_sys} and if there exists some $\epsilon > 0$ such that the following inequalities hold:
\begin{align}
\epsilon k_{pi} + k_{di} - \epsilon &> 0 \label{eq:DI_cond1}\\
\epsilon k_{pi}k_{di} - \epsilon^2k_{pi} - \frac{\epsilon}{4}k_{di}^2 &> 0 \label{eq:DI_cond2}\\
k_{pi} + 1 &> 0 \label{eq:DI_cond3}\\
k_{pi} - \epsilon^2 &> 0 \label{eq:DI_cond4}
\end{align}
where $k_1$ and $k_2$ are the elements of the matrix $K_{pd}$ in the parent controller, and $k_{pi}$ and $k_{di}$ are the $i$th diagonal elements of the $K_p$ and $K_d$ matrices, respectively.

The origin is stable with the same region of attraction as given by Theorem \ref{thm:LQR_SI_sys}.
\end{theorem}
\begin{proof}
See Appendix \ref{apx:ProofLQR_DI}.
\end{proof}

\begin{theorem}
\label{thm:ARISE_DI_sys}
Given the error vector $\mathbf{z} = \begin{bmatrix}e_1 & e_2 & r & e_{\tau} & e_J\end{bmatrix}^T$, where $r = \dot{e}_2 + \alpha_2e_2$, the ARISE control law \eqref{eq:controlLaw}, \eqref{eq:controlLaw2}, \eqref{eq:controlLaw3}, and \eqref{eq:controlLaw4} on the parent system \eqref{eq:newDynamics} and the double integrator control law \eqref{eq:nodeControlLaw_double} on the swarm \eqref{eq:EOM_DI} asymptotically stabilize the origin of this system if the gains for the controller satisfy the same conditions on the ARISE controller from Theorem \ref{thm:ARISE_SI_sys} and there exists some $\epsilon > 0$ that satisfy the inequalities listed in Theorem \ref{thm:LQR_DI_sys}.

This system has the same region of attraction as given by Theorem \ref{thm:ARISE_SI_sys}.
\end{theorem}
\begin{proof}
See Appendix \ref{apx:ProofARISE_DI}.
\end{proof}

\subsection{Heterogeneous Swarm}
The final swarm we consider is a heterogeneous swarm consisting of single integrator members as defined in \eqref{eq:swarmModel} as well as double integrator members as defined in \eqref{eq:EOM_DI}.  To control the swarm, we simply use the control laws from \eqref{eq:nodeControlLaw} and \eqref{eq:nodeControlLaw_double} for single integrator members and double integrator members respectively.

We demonstrate the stability of this controller with both parent controllers in the following theorems:
\begin{theorem}
\label{thm:LQR_noHomo_sys}
Given the error vector $\mathbf{z} = \begin{bmatrix}\theta & \dot{\theta} & e_\tau & e_J \end{bmatrix}^T$, the PD control law \eqref{eq:ControlLawPD} on the parent system \eqref{eq:newDynamics} and the single integrator control law \eqref{eq:nodeControlLaw} for single integrator members of the swarm and the double integrator control law \eqref{eq:nodeControlLaw_double} for double integrator members of the swarm asymptotically stabilizes the origin of this system if the gains on the parent controller satisfy the same conditions as stated in Theorem \ref{thm:LQR_SI_sys}, if $K > 0$ for the single integrator controllers, and there exists some $\epsilon > 0$ that satisfies the conditions given in Theorem \ref{thm:LQR_DI_sys}.

The origin is stable with the same region of attraction as given by Theorem \ref{thm:LQR_SI_sys}.
\end{theorem}
\begin{proof}
See Appendix \ref{apx:ProofLQR_noHomo}.
\end{proof}

\begin{theorem}
\label{thm:ARISE_noHomo_sys}
Given the error vector $\mathbf{z} = \begin{bmatrix}e_1 & e_2 & r & e_{\tau} & e_J\end{bmatrix}^T$ where $r = \dot{e}_2 + \alpha_2e_2$, the ARISE control law \eqref{eq:controlLaw} \eqref{eq:controlLaw2} \eqref{eq:controlLaw3}, and \eqref{eq:controlLaw4} on the parent system \eqref{eq:newDynamics} and the single integrator control law \eqref{eq:nodeControlLaw} for single integrator members of the swarm and the double integrator control law \eqref{eq:nodeControlLaw_double} for double integrator members of the swarm asymptotically stabilizes the origin of this system if the gains on the parent controller satisfy the same conditions as stated in Theorem \ref{thm:ARISE_SI_sys} and the controllers for the single and double integrator portions of the swarm satisfy the conditions given in Theorems \ref{thm:ARISE_SI_sys} and \ref{thm:ARISE_DI_sys} respectively.

This system has the same region of attraction as given by Theorem \ref{thm:ARISE_SI_sys}.
\end{theorem}
\begin{proof}
See Appendix \ref{apx:ProofARISE_noHomo}.
\end{proof}

Note that this control law is independent of the size and composition of the swarm.

\section{Simulation Results}
\label{sec:simExp}
\begin{table}
	\centering
	\caption{Table of physical parameters of simulated system}
	\begin{tabular}{| c | l | c | l |}
		\hline
		Parameter 	& Value 			& Parameter	& Value				\\
		\hline
		$m_{max}$	& $0.75kg$			& $\gamma_1$	& $0.01N$		\\
		$m_{min}$	& $0.25kg$			& $\gamma_2$	& $1000s^2$		\\
		$c_{max}$	& $1.5N \cdot s/m$	& $\gamma_3$	& $700s^2$		\\
		$c_{min}$	& $0.5N \cdot s/m$	& $\gamma_4$	& $0.02N$		\\
		$J$			& $0.5kg \cdot m^2$	& $\gamma_5$	& $1000s^2$		\\
		$L$			& $1m$				& $\gamma_6$	& $1N \dot s$	\\
		\hline
	\end{tabular}
	\label{tab:PhysicalParameters}
\end{table}

\begin{figure}
\centering
\includegraphics[width=6in]{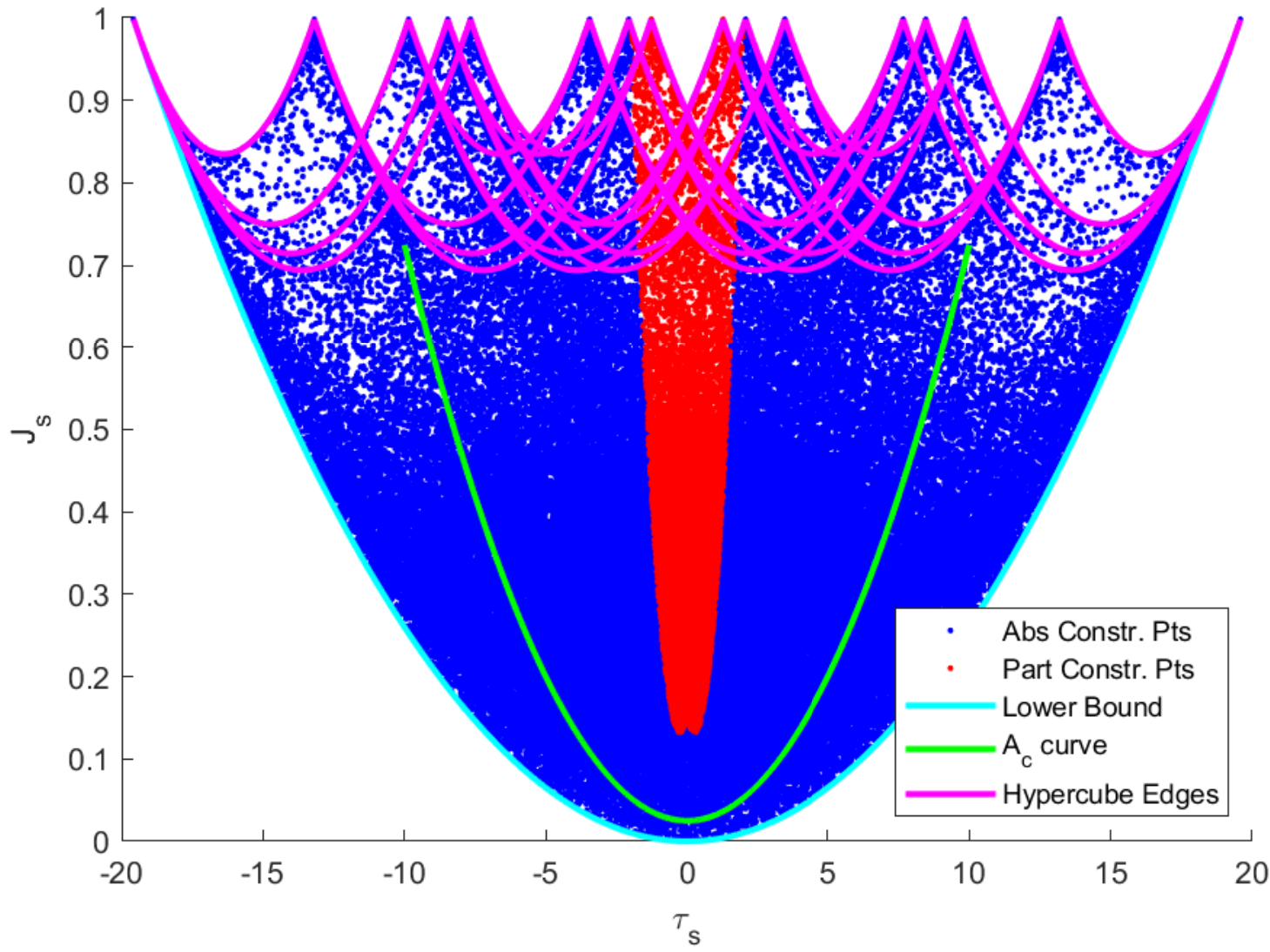}
\caption{Results of mapping the vertices and edges of the 4-dimensional hypercube into the abstract space. The magenta lines are the mappings of the edges of this hypercube.  The cyan line is the lower bound of the abstract domain.  The  dots (absolutely constrained states) and red dots (partially constrained states) are the Monte Carlo mapping used to confirm the absolutely constrained region.  The green curve is the surface $A_c$ on which the desired abstract state is constrained by \eqref{eq:JLimit} and \eqref{eq:tauLimit}.}
\label{fig:posMapping}
\end{figure}

We demonstrate the example controllers developed in the previous section in simulation using multiple different sized swarms. The physical parameters of the system are given in Table \ref{tab:PhysicalParameters}.  The masses of the child systems are uniformly distributed between $m_{min}$ and $m_{max}$.  The linear damping constants of the double integrator child systems are uniformly distributed between $c_{min}$ and $c_{max}$. To compare results, we predetermine the masses and damping coefficients for the 4 child systems used in the following simulations.  These values are $m_1 = 0.3552kg$, $m_2 = 0.3532kg$, $m_3 = 0.6762kg$, $m_4 = 0.4596kg$, $c_1 = 0.7290Ns/m$, $c_2 = 1.4133Ns/m$,  $c_3 = 0.6524Ns/m$, and  $c_4 = 1.3258Ns/m$.  These swarms have a starting position of $\mathbf{p}_0=\begin{bmatrix} 0.125m & -0.125m & 0.125m & -0.125m \end{bmatrix}^T$ with an initial velocity of $0$.

Figure \ref{fig:posMapping} shows the abstract space divided into its absolutely constrained, partially constrained, and unconstrained regions.  It also shows the manifold on which our desired abstract state moves as well as the edges of the hypercube that bounds the swarm state.  The chosen manifold, which lies in the interior of the absolutely constrained space, is given by
\begin{align}
J_{sd} &= 0.0125\tau_{sd}^2 + 0.025 \label{eq:JLimit} \\
\abs{\tau_{sd}} &\leq \tau_{max} \label{eq:tauLimit}
\end{align}

We now apply the various controllers designed in section \ref{sec:exCase} to this system in MATLAB simulation and present the results.

\subsection{PD Controller}
\begin{figure*}
\centering
\includegraphics[width=6.5in,height=3in]{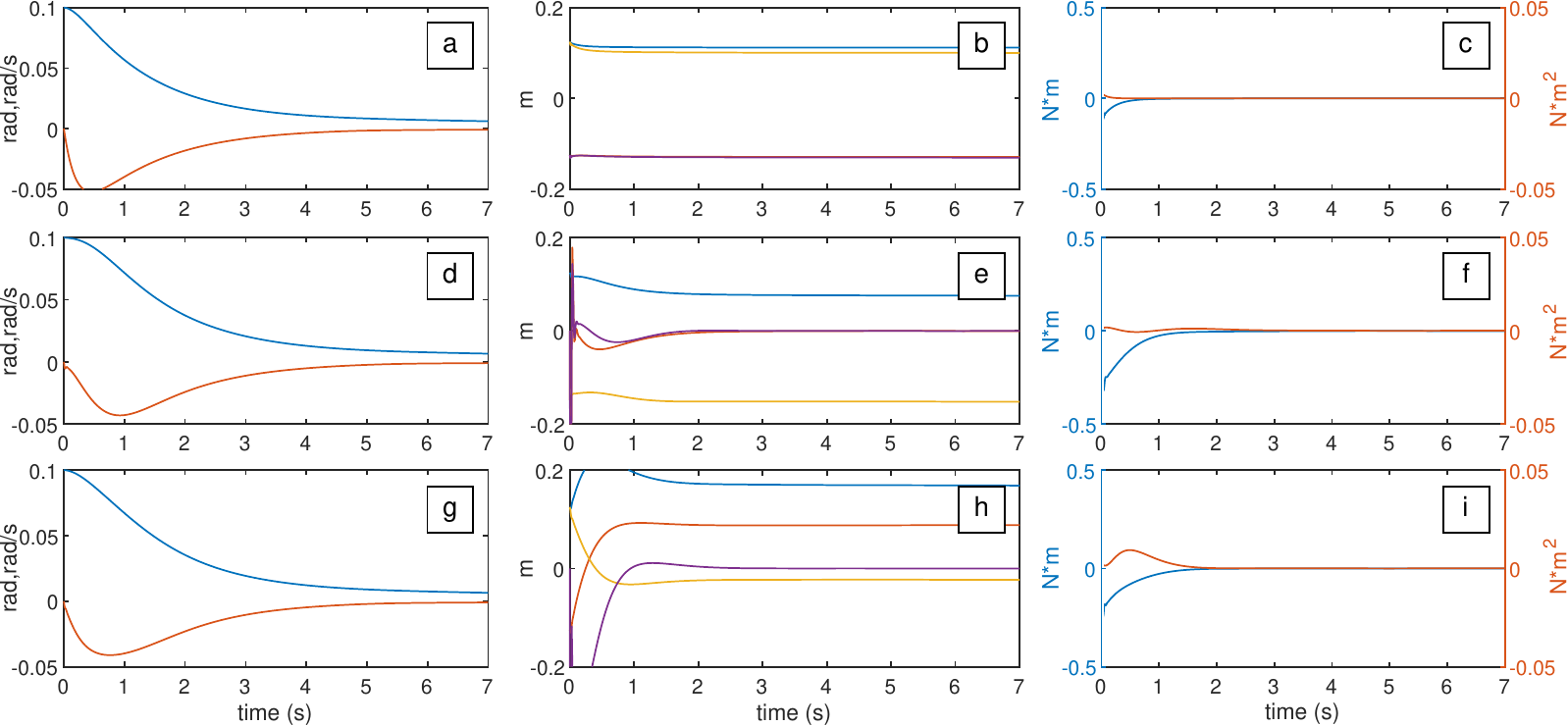}
\caption{Results from the simulations of PD controller.  In the left column is the parent state, consisting of $\theta$ (blue) and $\omega$ (red).  The center column shows the positions of the swarm members, and the right column shows the error in the abstract space consisting of the error in $\tau$ (blue) and the error in $J_s$ (red).  The top row is the simulation of the single integrator swarm, the center row is the simulation of the double integrator swarm, and the bottom row is the simulation of the heterogeneous swarm.}
\label{fig:PD_res}
\end{figure*}

We implement the PD parent controller on the three different swarms using $K=10I$, $K_p=10I$, $K_d=5I$, $k_{sd}=1$, and $\tau_{max}=5$.  These are LQR gains calculated using $Q = diag\left\lbrace 10,\ 1\right\rbrace$ and $R=1$.  Each swarm consists of 4 robots with the mass, damping coefficients, and initial condition listed previously.  The initial angle of the plane is $0.1rad$ with $0$ velocity.  The controller gain for the parent system is calculated to be $K_{pd}=\begin{bmatrix} 3.1623 & 3.2859\end{bmatrix}$. The stability of the single integrator swarm can be analyzed using Theorem \ref{thm:LQR_SI_sys}.  Gain $k_1$ is positive, and $K$ is a positive definite, symmetric matrix.  With $\theta_{max} = 0.2$, we can calculate the region of attraction to be where $\norm{\mathbf{z}}^2 < 0.0326$.  We then consider Figure \ref{fig:posMapping} and determine acceptable constraints on the system to be $\abs{e_J} \leq 0.03$ and $\abs{e_{\tau}} < 2$.  We find that the constrained region of attraction is $\norm{\mathbf{z}}^2 \leq 0.0326$, and so the initial errors $\mathbf{z}_0 = \begin{bmatrix} 0.1 & 0 & 0.0273 & 0.0288 \end{bmatrix}$ fall within this domain.  We can calculate $\dot{J}_{max}=1.9059$ and show that $k_2 > \frac{\dot{J}_{max}}{2\cos\left(\theta_{max}\right)}$, and thus all the conditions for stability are satisfied.  Theorem \ref{thm:LQR_DI_sys} can be used to show that the double integrator swarms has the same region of attraction, and thus the heterogeneous swarm does too.

Figure \ref{fig:PD_res} shows the results of the simulations of the PD parent controller with all three swarm types.  The left column of plots show the parent state, the center column shows the positions of all the child systems in the swarm, and the right column shows the errors in the abstract state.  All simulations have the same initial conditions. The top row of plots is the results from the single integrator system, the middle row is the double integrator system, and the bottom row is the heterogeneous system.  The second-order system has a clearly second-order response with some overshoot in the swarm positions due to the poles of the linearization of the double integrator system being complex. In the heterogeneous system, there is some oscillatory behavior in the swarm positions as well but not as much as in the swarm consisting of only double integrators.  However, the responses of the parent system for all three swarm types are very similar.

\subsection{ARISE Controller}
\begin{figure*}
\centering
\includegraphics[width=6.5in,height=3in]{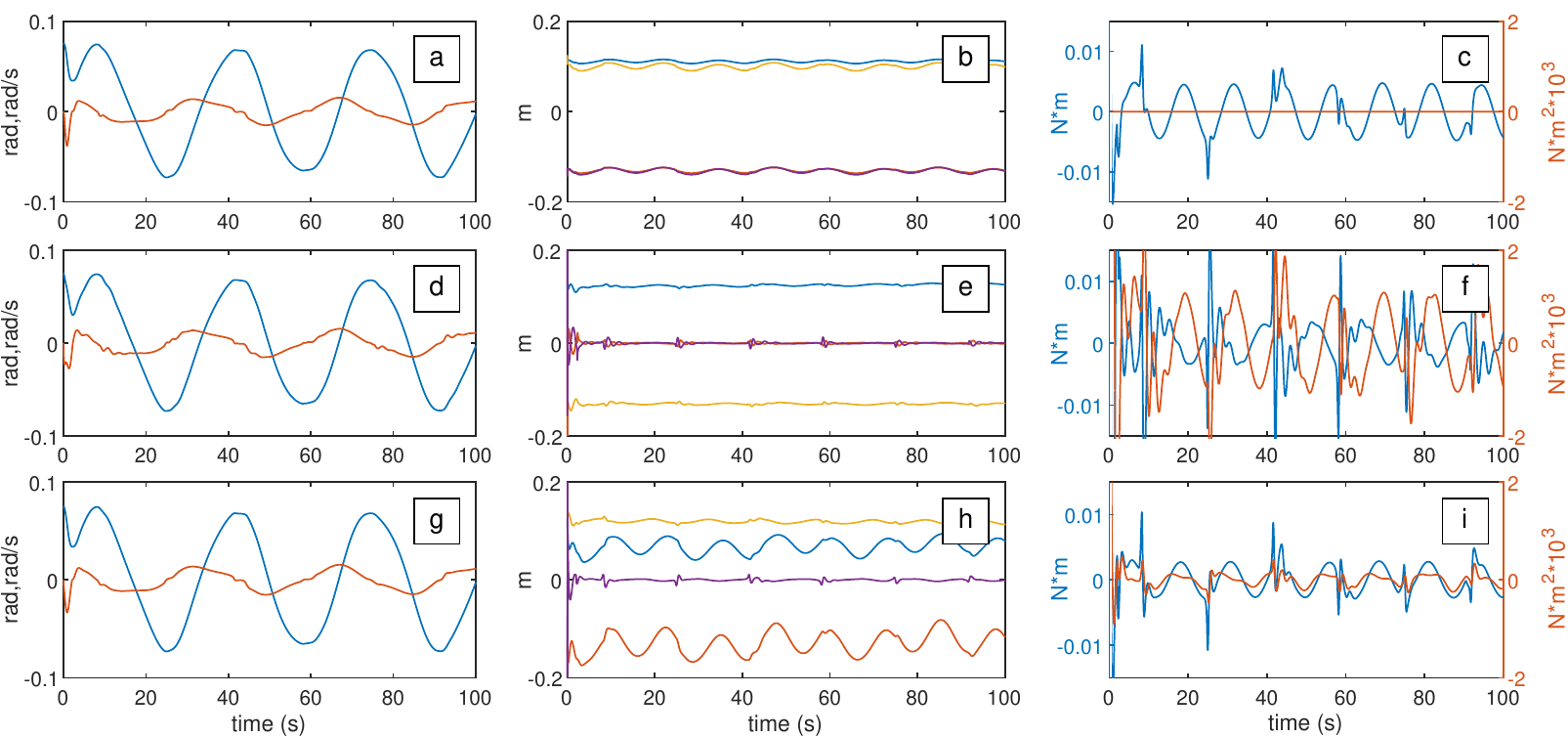}
\caption{Results from the simulations of ARISE controller.  These plots are arranged in the same way as Figure \ref{fig:PD_res}.}
\label{fig:ARISE_res}
\end{figure*}

Next, we apply the ARISE controller to parent system using the three different swarms.  The ARISE controller is implemented with $k_s=1$, $\alpha_1=1$, $\alpha_2=2$, $\beta=0.5$, and $\Gamma=diag\left(10, 1, 1, 10\right)$, while the swarm system controllers use the same gains as the PD parent controller case.  For all of these experiments, the ARISE controller attempts to track a reference trajectory $\theta_d\left(t\right) = 0.7\sin\left(0.015\pi t\right)$ in the presence of a sinusoidal disturbance $\tau_d\left(t\right) = 0.1\sin\left(0.5t\right)$.  The initial state of the plane is $0.075rad$ with 0 velocity. For the single integrator system, Theorem \ref{thm:ARISE_SI_sys} can be used to determine that the constrained region of attraction is where $\norm{\mathbf{z}} \leq 0.0394$.  The initial states are $\mathbf{z}_{p0} = \begin{bmatrix} -0.075 & -0.0618 & -0.1104\end{bmatrix}^T$, $\mathbf{y}_{p0} = \begin{bmatrix} 0.0434 & 0.001 & 0.01 & 0.1 \end{bmatrix}^T$, and $\mathbf{z}_{a0} = \begin{bmatrix} 0.0273 & 0.0288 \end{bmatrix}^T$. These initial conditions fall within this region of attraction; thus, the error system converges to $0$ so long as $c_{max}$ remains less than $k_s$.  Theorem \ref{thm:ARISE_DI_sys} can be used to show that the double integrator swarm has the same region of attraction, and thus, so does the heterogeneous swarm.

Figure \ref{fig:ARISE_res} shows the result for the ARISE controller on all three swarms.  This figure is organized in the same manner as Figure \ref{fig:PD_res}.  The double integrator swarm, like with the PD controller, has more oscillation than the single integrator.  However, unlike the PD controller, the heterogeneous system has greater oscillations.  There is also more error present in the abstract state due to the sinusoidal disturbance.

\begin{figure*}
\centering
\includegraphics[width=6.5in,height=1in]{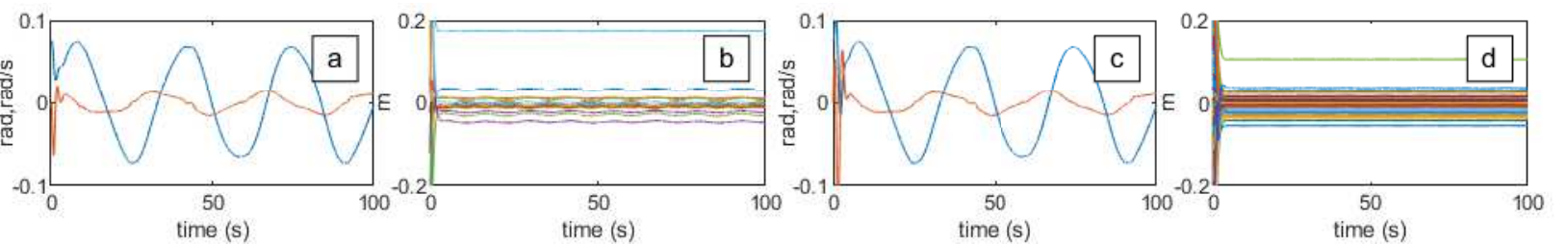}
\caption{Results from the simulations of ARISE controller with larger, heterogeneous swarms.  Plots a and b show, respectively, the parent state and the swarm positions for a swarm of 10 single integrators and 10 double integrators.  Plots c and d show the same for a swarm of 100 single integrators and 100 double integrators.}
\label{fig:ARISE_lg_swarm}
\end{figure*}

In addition to allowing different controllers to be used for both the parent and swarm systems, the abstraction is also independent of the size of the swarm.  Figure \ref{fig:ARISE_lg_swarm} shows the results for the ARISE controller on heterogeneous swarms of 20 and 200 robots.  The parent system response in these plots is very similar to the responses in the other example of the ARISE controller, the only difference being the way the swarm moves and some transients that get amplified at the beginning.  Note that the larger the swarm is, the tighter its final formation due to the controller driving the swarm to the same desired moment of inertia with more total mass.  Also due to the larger mass, less movement is needed within the swarm to effect the same gravity torque, resulting in less swarm motion.

\section{Conclusion}
\label{sec:conclusion}

We propose a new approach for designing controllers for a class of systems where a swarm of subsystems acts upon a larger parent system with its own dynamics.  We condense the interactions between the parent system into an abstract state of the swarm, which allows controllers for the parent system to be designed independently of the swarm.  We also consider how constraints on the swarm states map to constraints in the abstract space.  We then develop stability proofs with and without state constraints based on the abstract state rather than the full state of the swarm.

To validate these results, we present an example case of a passive, dynamic parent system being manipulated by a large number of child subsystems.  We demonstrate the modularity of this approach by using several different controllers for the parent and swarm systems in combination.  We also demonstrate with these simulations that the size and composition of the swarm do not significantly impact the performance of the controller on the parent system, and that the controllers for the child systems can used independent of the size of the swarm.  This architecture can be further extended to a series or hierarchy of cascaded systems with abstractions modeling the coupling between layers.

\appendices

\section{Proof of Theorem \ref{thm:LQR_SI_sys}}
\label{apx:ProofLQR_SI}
\begin{proof}
Let the desired abstract state lie on a continuous manifold $J_{sd}\left(\tau_{d}\right)$.  Let $\abs{\tau_d} \leq \tau_{max}$, and let $J_{sd,min} \leq J_{sd}(\tau_d) \leq J_{sd,max}$, where $J_{sd}\left(0\right) = J_{sd,min}$.

The equation of motion \eqref{eq:newDynamics} is linearized about the origin, including neglecting the nonlinear friction terms.  This results in the following linear system:
\begin{equation}
\dot{\mathbf{x}_p} = \begin{bmatrix} \dot{\theta} \\ \ddot{\theta}\end{bmatrix} = \begin{bmatrix} 0 & 1 \\ 0 & \frac{\gamma_6}{J + J_{s,min}}\end{bmatrix}\begin{bmatrix}\theta \\ \dot{\theta}\end{bmatrix} + \begin{bmatrix}0 \\ \frac{-1}{J + J_{s,min}}\end{bmatrix}\tau_{sd} = A\mathbf{x}_p + B\mathbf{u} 
\end{equation}
where $\mathbf{x}_p = \begin{bmatrix}\theta & \dot{\theta}\end{bmatrix}^T$ and $\mathbf{u} = \tau_s$.  We can find the controllability matrix of this system:
\begin{equation}
\mathcal{C} = \begin{bmatrix} B & AB \end{bmatrix} = \begin{bmatrix} 0 & \frac{-1}{J + J_{s,min}} \\ \frac{-\gamma_6}{\left(J + J_{s,min}\right)^2} & 0 \end{bmatrix}
\end{equation}
The controllability matrix $\mathcal{C}$ is full rank; thus, this system is controllable.  We can then compute the control law \eqref{eq:ControlLawPD} using an LQR method or pole placement, for example.

Let $V_p\left(\mathbf{z}_p,t\right)$ be a Lyapunov candidate for this system, where $\mathbf{z}_p = \begin{bmatrix}\theta & \dot{\theta}\end{bmatrix}^T$:
\begin{equation}
V_p\left(\mathbf{z}_p,t\right) = \frac{1}{2}\theta^2 + \frac{J + J_s\left(t\right)}{2}\dot{\theta}^2
\end{equation}
This function is positive definite with $\eta_1\norm{\mathbf{z}_p}^2 \leq V_p\left(\mathbf{z}_p,t\right) \leq \eta_2\norm{\mathbf{z}_p}^2$, where $\eta_1 = \min\left\lbrace\frac{1}{2}, \frac{J+J_{s,min}}{2}\right\rbrace$, $\eta_2 = \max\left\lbrace\frac{1}{2}, \frac{J + J_{s, max}}{2}\right\rbrace$. Thus $V_p$ is positive definite per Theorem 4.8 in \cite{Khalil2002}.

Taking the time derivative of $V_p$ and substituting \eqref{eq:newDynamics} and \eqref{eq:ControlLawPD} yields
\begin{equation}
\begin{split}
\begin{aligned}
\dot{V_p} &= \theta\dot{\theta} + \dot{\theta}\left(J + J_s\right)\ddot{\theta} + \frac{\dot{J}_s}{2}\dot{\theta}^2 \\
&=\theta\dot{\theta} + \dot{\theta}\left(-cos\left(\theta\right)\left(k_1\theta + k_2\dot{\theta}\right) - \dot{J}_s\dot{\theta} - f_f\left(\dot{\theta}\right)\right) + \frac{\dot{J}_s}{2}\dot{\theta}^2 \\
&=-\left(\cos\left(\theta\right)k_1 - 1\right)\theta\dot{\theta} - \left(\cos\left(\theta\right)k_2 + \frac{\dot{J}_s}{2}\right)\dot{\theta}^2 - \dot{\theta}f_f\left(\dot{\theta}\right) \\
&= -\mathbf{z}_p^T\begin{bmatrix}
0 & -1 \\ \cos\left(\theta\right)k_1 & \cos\left(\theta\right)k_2 + \frac{\dot{J}_s}{2}
\end{bmatrix}
\mathbf{z}_p - \dot{\theta}f_f\left(\dot{\theta}\right)
\end{aligned}
\end{split}
\end{equation}

We assume that $\abs{\theta} \leq \theta_{max} < \frac{\pi}{2}$. By Lemma \ref{lm:J_dot_is_bounded} in Appendix \ref{apx:lemmas}, $\abs{\dot{J}_s}$ is bounded above; thus, define domain $D_{Jmax}$ on which $\abs{\dot{J}_s} < \dot{J}_{max}$.  Therefore,
\begin{equation}
\dot{V_p} \leq -\mathbf{z}_p^T\begin{bmatrix}0 & -1 \\ \cos\left(\theta_{max}\right) k_1 & \cos\left(\theta_{max}\right)k_2 - \frac{\dot{J}_{max}}{2}\end{bmatrix}\mathbf{z}_p \\
- \dot{\theta}f_f\left(\dot{\theta}\right) \equiv  -\mathbf{z}_p^T M \mathbf{z}_p - \dot{\theta}f_f\left(\dot{\theta}\right)
\end{equation}
The term $\dot{\theta}f_f\left(\dot{\theta}\right)$ is positive definite since if $\dot{\theta} < 0$, then $f_f\left(\dot{\theta}\right) < 0$ and if $\dot{\theta} > 0$, then $f_f\left(\dot{\theta}\right) > 0$.

The $\mathbf{z}_p^T M \mathbf{z}_p$ term can be bounded by
\begin{equation}
\lambda_{min}\norm{\mathbf{z}_p}^2 \leq \mathbf{z}_p^T M \mathbf{z}_p \leq \lambda_{max}\norm{\mathbf{z}_p}^2
\end{equation}
where $\lambda_{min}$ and $\lambda_{max}$ are the minimum and maximum eigenvalues of $M$, respectively.  If these eigenvalues are both positive, then this term is positive definite.  We find the characteristic equation of the matrix M:
\begin{equation}
\lambda^2 - \left(\cos\left(\theta_{max}\right)k_2 - \frac{\dot{J}_{max}}{2}\right)\lambda + \cos\left(\theta_{max}\right)k_1 = 0
\end{equation}
using Routh-Hurwitz Stability Criterion \cite{Franklin2010}, the following need to be true for the eigenvalues to be positive:
\begin{align}
\cos\left(\theta_{max}\right)k_2 - \frac{\dot{J}_{max}}{2} &> 0 \\
\cos\left(\theta_{max}\right)k_1 &> 0
\end{align}
Therefore, if $k_1 > 0$ and $k_2 > \frac{\dot{J}_{max}}{2\cos\left(\theta_{max}\right)}$, then $\dot{V}_p$ is negative definite on the domain
\begin{equation}
D = \left\lbrace\mathbf{z} \in D_{Jmax}\ \vert\ \abs{\theta} \leq \theta_{max}\right\rbrace
\end{equation}

We now consider the swarm system.  We propose a quadratic Lyapunov function
\begin{equation}
V_a = \frac{1}{2}\mathbf{z}_a^T\mathbf{z}_a
\end{equation}
where
\begin{equation}
\mathbf{z}_a = \begin{bmatrix} e_\tau & e_J \end{bmatrix}^T
\end{equation}
This function is clearly positive definite.  We then take its time derivative to arrive at
\begin{equation}
\begin{split}
\dot{V}_a &= \dot{\mathbf{z}}_a^T\mathbf{z}_a \\
&= \left(\dot{\mathbf{a}}_d - \left(K\mathbf{z}_a + \dot{\mathbf{a}}_d\right)\right)^T\mathbf{z}_a \\
&= -\mathbf{z}_a^T K^T \mathbf{z}_a \\
&= \dot{V}_a^\prime
\end{split}
\end{equation}
Therefore, if $K$ is positive definite, then $\dot{V}_a^\prime$ is negative definite on $D_a = \mathbb{R}^2$.

Define $V = V_p + V_a$.  By Theorem \ref{thm:GeneralStability}, the domain on which $\dot{V}$ is negative definite is
\begin{equation}
D = \left\lbrace\mathbf{z} \in D_{Jmax} \ \vert\ \abs{\theta} \leq \theta_{max} \right\rbrace
\end{equation}
The largest semilevel set that fits in this domain is
\begin{equation}
S_l = \left\lbrace\mathbf{z} \in \mathbb{R}^4\ \vert\ V\left(\mathbf{z}\right) \leq \frac{\theta_{max}^2}{2}\right\rbrace
\end{equation}
For simplicity, we can find the largest ball that fits into this semilevel set:
\begin{equation}
S = \left\lbrace\mathbf{z} \in \mathbb{R}^4\ \vert\ \norm{\mathbf{z}} \leq \frac{\theta_{max}^2}{2\eta}\right\rbrace
\end{equation}
where $\eta = max\left\lbrace \frac{1}{2},  \frac{J + J_{s, max}}{2} \right\rbrace$.  $S$ is a subset of the region of attraction of this system.  We can use this bound on $\norm{\mathbf{z}}$ to find $\dot{J}_{max}$ and thus the range of $k_2$ required to stabilize the origin of this system.
\end{proof}

\section{Proof of Theorem \ref{thm:ARISE_SI_sys}}
\label{apx:ProofARISE_SI}
\begin{proof}
We define a new error term $r$:
\begin{equation}
r = \dot{e}_2 + \alpha_2e_2
\end{equation}
My multiplying $r$ by the inertia term, we arrive at
\begin{equation}
\begin{split}
\left(J + J_s\right)r &= \left(J + J_s\right)\left(\ddot{e}_1 + \alpha_1\dot{e}_1 + \alpha_2 e_2\right) \\
&= \left(J + J_s\right)\left(\ddot{\theta}_d + \alpha_1\dot{e}_1 + \alpha_2 e_2\right) - \left(J + J_s\right)\ddot{\theta} \\
&= Y_d\lambda + W + cos\left(\theta\right)\tau_s + \tau_d \label{eq:Jr}
\end{split}
\end{equation}
where
\begin{equation}
W = \left(J + J_s\right)\left(\ddot{\theta}_d + \alpha_1\dot{e}_1 + \alpha_2 e_2\right) + \dot{J}_s\dot{\theta} + f_f\left(\dot{\theta}\right) - Y_d\lambda
\label{eqn:aux_W}
\end{equation}
We now rewrite the control laws \eqref{eq:controlLaw} and \eqref{eq:controlLaw3} in terms of $r$:
\begin{equation}
\tau_{sd} = - \sec\left(\theta\right)\left(Y_d\hat{\lambda} + \mu\right)
\label{eq:TauR}
\end{equation}
\begin{equation}
\dot{\hat{\lambda}} = \Gamma\dot{Y}_d^Tr
\label{eq:Lambda}
\end{equation}
where
\begin{equation}
\mu\left(t_0\right) = 0
\end{equation}
\begin{equation}
\hat{\lambda}\left(t_0\right) = \hat{\lambda}_0
\end{equation}
\begin{equation}
\dot{\mu} = \left(ks + 1\right)r + \beta sgn\left(e_2\right)
\label{eq:proof2eq1_mu_dot}
\end{equation}
and substitute \eqref{eq:proof2eq1_mu_dot}, \eqref{eq:abstractErrorDef}, and \eqref{eq:TauR} into \eqref{eq:Jr}:
\begin{equation}
\begin{split}
(J + J_s)r &= Y_d\lambda + W - Y_d\hat{\lambda} - \mu - \cos\left(\theta\right)e_\tau + \tau_d \\
&= Y_d\tilde{\lambda} + W - \mu - \cos\left(\theta\right)e_\tau + \dot{\tau}_d \label{eq:Jr2}
\end{split}
\end{equation}
where
\begin{equation}
\tilde{\lambda} = \lambda - \hat{\lambda}
\end{equation}

We next take the time derivative of \eqref{eq:Jr2}:
\begin{equation}
\left(J + J_s\right)\dot{r} = -\dot{J}_sr + \dot{Y}_d\tilde{\lambda} - Y_d\dot{\hat{\lambda}} + \dot{W} - \dot{\mu} + \dot{\theta}\sin\left(\theta\right)e_\tau - \cos\left(\theta\right)\dot{e}_\tau + \dot{\tau}_d
\end{equation}
and substitute \eqref{eq:TauR} and \eqref{eq:Lambda}:

\begin{equation}
\begin{split}
\left(J + J_s\right)\dot{r} &= -\dot{J}_sr + \dot{Y}_d\tilde{\lambda} - Y_d\Gamma\dot{Y}_d^Tr + \dot{W} - \left(k_s + 1\right)r - \beta sgn\left(e_2\right) + \dot{\theta}\sin\left(\theta\right) - \cos\left(\theta\right)\dot{e}_\tau + \dot{\tau}_d \\
&= -\frac{1}{2}\dot{J}_sr + \dot{Y}_d\tilde{\lambda} + N + E - \left(k_s + 1\right)r - \beta sgn\left(e_2\right) - e_2
\end{split}
\end{equation}
where
\begin{equation}
N = -Y_d\Gamma\dot{Y}_d^Tr + \dot{W} - \frac{1}{2}\dot{J}_sr + e_2 + \tau_d
\end{equation}
\begin{equation}
E = \dot{\theta}\sin\left(\theta\right)e_\tau - \cos\left(\theta\right)\dot{e}_\tau
\end{equation} 

Define
\begin{equation}
\tilde{N} = N - N_d
\end{equation}
\begin{equation}
\tilde{E} = E - E_d
\end{equation}
Then, by Lemmas \ref{lm:NisBounded} and \ref{lm:EisBounded},
\begin{equation}
\abs{\tilde{N}} \leq c_{max}\norm{\mathbf{z_p}}
\end{equation}
\begin{equation}
\abs{\tilde{E}} \leq \rho_{E} \left(\norm{\mathbf{z}_p}\right)\norm{\mathbf{z}_a}
\end{equation}
where
\begin{equation}
\mathbf{z}_p = \begin{bmatrix}e_1 & e_2 & r\end{bmatrix}^T
\end{equation}
\begin{equation}
\mathbf{z}_a = \begin{bmatrix} e_\tau & e_J \end{bmatrix}^T
\end{equation}

The last auxiliary functions we need are
\begin{equation}
P\left(t\right) = \beta\abs{e_2\left(t_0\right)} - e_2\left(t_0\right)\left(N_d\left(t_0\right) + E_d\left(t_0\right)\right) - \displaystyle\int_{t_0}^tL\left(\sigma\right)d\sigma
\end{equation}
where
\begin{equation}
L\left(t\right) = r\left(\left(N_d\left(t\right) + E_d\left(t\right)\right) - \beta sgn\left(e_2\left(t\right)\right)\right)
\end{equation}

We now define a new vector $\mathbf{y_p}$ that contains terms that are bounded but do not necessarily converge to zero, such as error in the adaptive terms and integral terms:
\begin{equation}
\mathbf{y_p} = \begin{bmatrix}\tilde{\lambda}^T & \sqrt{P}\end{bmatrix}^T
\end{equation}
and a Lyapunov candidate
\begin{equation}
V_p = e_1^2 + \frac{1}{2}e_2^2 + \frac{1}{2}\left(J + J_s\right)r^2 + P + \frac{1}{2}\tilde{\lambda}^T\Gamma^{-1}\tilde{\lambda}
\label{eq:V_PAppB}
\end{equation}
By Lemma \ref{lm:PIsPosDef}, $\int_{t_0}^{t}L\left(\sigma\right)d\sigma \leq \beta\abs{e_2\left(t_0\right)} - e_d\left(t_0\right)\left(N_d\left(t_0\right) + E_d\left(t_0\right)\right)$.  Thus we can show that $P \geq 0$, and so $\sqrt{P} \in \mathbb{R}$ and
\begin{equation}
\eta_1\norm{\begin{bmatrix}\mathbf{z}_p^T & \mathbf{y}_p^T\end{bmatrix}^T} \leq V_p \leq\eta_2\norm{\begin{bmatrix}\mathbf{z}_p^T & \mathbf{y}_p^T\end{bmatrix}^T}
\end{equation}
where
\begin{equation}
\eta_1 = \frac{1}{2}\min\left\lbrace 1, J + J_{s,min}, \lambda_{min}\left(\Gamma^{-1}\right)\right\rbrace
\end{equation}
\begin{equation}
\eta_2 = \frac{1}{2}\max\left\lbrace 2, J + J_{s,max}, \lambda_{max}\left(\Gamma^{-1}\right)\right\rbrace
\end{equation}
We can prove this by writing $V_p = \begin{bmatrix} \mathbf{z}_p^T & \mathbf{y}_p^T \end{bmatrix}Q\begin{bmatrix} \mathbf{z}_p^T & \mathbf{y}_p^T \end{bmatrix}^T$, where $Q$ is a block diagonal matrix in which each of the terms in \eqref{eq:V_PAppB} contributes a block.  All of the blocks are 1-dimensional except for $\Gamma$.  $\eta_1$ and $\eta_2$ are the minimum and maximum eigenvalues of $Q$, which are positive if $\Gamma$ is positive definite; thus our Lyapunov candidate is positive definite.  Note also that Lemma \ref{lm:PIsPosDef} gives the lower bound on $\beta$.

We then take the time derivative and substitute the control laws \eqref{eq:controlLaw} and \eqref{eq:controlLaw3} to show that $\dot{V}_p$ is negative semi-definite:
\begin{equation}
\begin{split}
\dot{V_p} &= 2e_1\dot{e}_1 + e_2\dot{e}_2 + \left(J + J_s\right)\dot{r}r + \frac{1}{2}\dot{J}_sr^2 + \dot{P} - \tilde{\lambda}\Gamma^{-1}\dot{\hat{\lambda}} \\
&= 2e_1\left(e_2 - \alpha_1e_1\right) + e_2\left(r - \alpha_2e_2\right) + -\frac{1}{2}\dot{J}_sr^2 + \dot{Y}_d\tilde{\lambda}r + \left(N + E\right)r - \left(k_s + 1\right)r^2 \\
&\phantom{=} - \beta sgn\left(e_2\right)r - e_2r + \frac{1}{2}\dot{J}_sr^2 - L - \tilde{\lambda}^T\Gamma^{-1}\Gamma\dot{Y}_d^Tr \\
&= - 2\alpha_1e_1^2 - \alpha_2e_2^2 - \left(k_s + 1\right)r^2 + 2e_1e_2 + \left(\tilde{N} + \tilde{E}\right)r + \left(N_d + E_d\right)r - \beta sgn\left(e_2\right)r \\
&\phantom{=} - r\left(N_d + E_d - \beta sgn\left(e_2\right)\right)
\\
&= - 2\alpha_1e_1^2 - \alpha_2e_2^2 - \left(k_s + 1\right)r^2 + 2e_1e_2 + \left(\tilde{N} + \tilde{E}\right)r \\
&\leq -\left(2\alpha_1 - 1\right)e_1^2 - \left(\alpha_2 - 1\right)e_2^2 - \left(k_s + 1\right)r^2 + \left(\tilde{N} + \tilde{E}\right)r \\
&\leq -\eta_3\norm{\mathbf{z}_p}^2 - \left(k_sr^2 - c_{max}\norm{\mathbf{z}_p}\abs{r}\right) + \rho_{E}\left(\norm{\mathbf{z}_p}\right)\norm{\mathbf{z}_a}\norm{\mathbf{z}_p}
\label{eq:proof1}
\end{split}
\end{equation}
where $\eta_3 = min\left\lbrace 2\alpha_1 - 1, \alpha_2 - 1, 1 \right\rbrace$. Thus $\alpha_1 > \frac{1}{2}$ and $\alpha_2 > 1$ make $\eta_3$ positive and thus make $\dot{V}_p^\prime$ negative definite.  We break this into its uncoupled and coupled terms per Theorem \ref{thm:GeneralStability}:
\begin{align}
{V}_p^\prime &= -\eta_3\norm{\mathbf{z}_p}^2 - \left(k_sr^2 - c_{max}\norm{\mathbf{z}_p}\abs{r}\right) \\
\dot{V}_{cp} &= \rho_{E}\left(\norm{\mathbf{z}_p}\right)\norm{\mathbf{z}_a}\norm{\mathbf{z}_p}
\end{align}

We complete the square on \eqref{eq:proof1} to get
\begin{equation}
\begin{split}
\begin{aligned}
\dot{V}_p^\prime &\leq -\eta_3\norm{\mathbf{z}_p}^2 - \left(\left(\sqrt{k_s}\abs{r} - \frac{c_{max}\norm{\mathbf{z}_p}}{2\sqrt{k_s}}\right)^2 - \frac{c_{max}^2\norm{\mathbf{z}_p}^2}{4k_s}\right) \\
&\leq -\eta_3\norm{\mathbf{z}_p}^2 + \frac{c_{max}^2\norm{\mathbf{z}_p}^2}{4k_s} - \left(\sqrt{k_s}\abs{r} - \frac{c_{max}\norm{\mathbf{z}_p}}{2\sqrt{k_s}}\right)^2 \\
&\leq -\left(\eta_3 - \frac{c_{max}}{4k_s}\right)\norm{\mathbf{z}_p}^2
\end{aligned}
\end{split}
\end{equation}
which is negative definite if $k_s > \frac{c_{max}}{4\eta_3}$ on $D_p = \mathbb{R}^3$.

We use the same Lyapunov function for the swarm as Theorem \ref{thm:LQR_SI_sys}:
\begin{equation}
V_a = \frac{1}{2}\mathbf{z}_a^T\mathbf{z}_a
\end{equation}

With these two Lyapunov functions, $V_p$ and $V_a$, we can say that by Theorem \ref{thm:GeneralStability}, $\mathbf{z}_p$ and $\mathbf{z}_a$ asymptotically converge to the origin and $\mathbf{y}_p$ is bounded, so long as
\begin{equation}
\left(\eta_3 - \frac{c_{max}}{4k_s}\right)\norm{\mathbf{z}_p}^2 + \mathbf{z}_a^T K^T \mathbf{z}_a > \rho_{E}\left(\norm{\mathbf{z}_p}\right)\norm{\mathbf{z}_p}\norm{\mathbf{z}_a} \label{eq:condition1}
\end{equation}

Let $\mathbf{z}^{\prime} = \begin{bmatrix} \mathbf{z}_p^T & \mathbf{z}_a^T \end{bmatrix}^T$.  We specify a domain $D$ on which \eqref{eq:condition1} is satisfied:
\begin{equation}
D = \left\lbrace \mathbf{z} \in D_p \times D_a \mid  \left(\eta_3 - \frac{c_{max}}{4k_s}\right)\norm{\mathbf{z}_p}^2 + \mathbf{z}_a^T K^T \mathbf{z}_a > \rho_{E}\left(\norm{\mathbf{z}_p}\right)\norm{\mathbf{z}_p}\norm{\mathbf{z}_a} \right\rbrace
\end{equation}
By Theorem \ref{thm:GeneralStability}, we can find a region of attraction $S$ in which the initial state must be for the system to converge.

We can see that
\begin{equation}
\left(\eta_3 - \frac{c_{max}}{4k_s}\right)\norm{\mathbf{z}_p}^2 + \mathbf{z}_a^T K^T \mathbf{z}_a \geq \left(\eta_3 - \frac{c_{max}}{4k_s}\right)\norm{\mathbf{z}}^2 + \lambda_{min}\norm{\mathbf{z}}
\end{equation}
and
\begin{equation}
\rho_{E}\left(\norm{\mathbf{z}_p}\right)\norm{\mathbf{z}_p}\norm{\mathbf{z}_a} \leq \frac{1}{2}\rho_E\left(\norm{\mathbf{z}_p}\right)\norm{\mathbf{z}}^2
\end{equation}
where $\lambda_{min}$ in the minimum eigenvalue of $K$. Thus we can find the largest ball $B$ within the domain $D$:
\begin{equation}
B = \left\lbrace \mathbf{z} \in D\ \vert\ \norm{\mathbf{z}} \leq \rho_E^{-1}\left(2\eta_3 - \frac{c_{max}}{2k_s} + 2\lambda_{min}\right) \right\rbrace
\end{equation}
The boundary of this ball is much easier to compute than the boundary of $D$; thus, we can use its boundary to find a semilevel set $S$ that is contained within the region of attraction of this system:
\begin{equation}
S_l = \left\lbrace\mathbf{z} \in B,\ \mathbf{y}_p \in \mathbb{R}^2\ \vert\  V_p\left(\mathbf{z}, \mathbf{y}_p\right) + V_a\left(\mathbf{z}\right) \leq \eta_1\rho^2\right\rbrace
\end{equation}
where $\rho = \rho_E^{-1}\left(2\eta_3 - \frac{c_{max}}{2k_s} + 2\lambda_{min}\right)$.

Like in the proof for Theorem \ref{thm:LQR_SI_sys}, we can find the largest ball within this set to simplify our calculations later:
\begin{equation}
S = \left\lbrace \mathbf{z} \in B,\ \mathbf{y}_p \in \mathbb{R}^2\ \vert\  \norm{\begin{bmatrix} \mathbf{z}^T & \mathbf{y}_p^T \end{bmatrix}^T} \leq \frac{\eta_1}{\eta_2}\rho^2 \right\rbrace
\end{equation}
Thus, $S$ is a subset of the region of attraction of this system.
\end{proof}

\section{Proof of Theorem \ref{thm:LQR_DI_sys}}
\label{apx:ProofLQR_DI}
\begin{proof}
First we propose a Lyapunov candidate for the swarm system:
\begin{equation}
V_a = \frac{1}{2}\mathbf{z}_a^T \begin{bmatrix} K_p & \epsilon I \\ \epsilon I & I \end{bmatrix} \mathbf{z}_a \equiv \frac{1}{2}\mathbf{z}_a^T M_1 \mathbf{z}_a
\end{equation}
where $\mathbf{z}_a = \begin{bmatrix}\mathbf{e}_a^T & \dot{\mathbf{e}}_a^T\end{bmatrix}^T$.

We can say that $\lambda_{min}\norm{\mathbf{z}_a}^2 \leq V_a \leq \lambda_{max}\norm{\mathbf{z}_a}^2$ where $\lambda_{min}$ and $\lambda_{max}$ are the minimum and maximum eigenvalues of $M_1$, respectively.  Therefore, this function is positive definite if all the eigenvalues of $M_1$ are positive.  The eigenvalues of $M_1$ can be found from
\begin{equation}
\begin{split}
\det\left(M_1 - \lambda I\right) &= \det\left(\left(K_p - \lambda I\right)\left(1 - \lambda\right)I - \epsilon^2I\right) \\
&= \det\left(\lambda^2I - \left(K_p + I\right)\lambda + K_p - \epsilon^2I\right) \\
&\begin{aligned}
 = &\left(\lambda^2 - \left(k_{p1} + 1\right)\lambda + k_{p1} - \epsilon^2\right) \left(\lambda^2 - \left(k_{p2} + 1\right)\lambda + k_{p2} - \epsilon^2\right)
\end{aligned}
\end{split}
\end{equation}
where $k_{pi}$ is the $i$th diagonal element of $K_p$.  From here we can see that to have positive eigenvalues, the following must be true:
\begin{equation}
k_{pi} - \epsilon^2 > 0
\end{equation}
for $i = \left\lbrace 1,\ 2 \right\rbrace$.

Next we consider the time derivative of the Lyapunov candidate:
\begin{equation}
\begin{split}
\dot{V}_a &= \mathbf{z}_a^TM_1\dot{\mathbf{z}}_a \\
&= \mathbf{e}_a^TK_p\dot{\mathbf{e}}_a + \epsilon\mathbf{e}_a^T\ddot{\mathbf{e}}_a + \epsilon\dot{\mathbf{e}}_a^T\dot{\mathbf{e}}_a + \dot{\mathbf{e}}_a^T\ddot{\mathbf{e}}_a
\end{split}
\end{equation}
We then substitute the swarm equation of motion \eqref{eq:EOM_DI} and control law \eqref{eq:swarmControlLaw_double}:
\begin{equation}
\begin{split}
\dot{V}_a &= \mathbf{e}_a^TK_p\dot{\mathbf{e}}_a + \epsilon\dot{\mathbf{e}}_a^T\dot{\mathbf{e}}_a + \left(\epsilon\mathbf{e}_a^T + \dot{\mathbf{e}}_a^T\right)\left( \ddot{\mathbf{a}}_d - \dot{\Phi}\dot{\mathbf{p}} \right. \\
&\phantom{=} - \Phi M\left(M^{-1}\Phi^{\dagger}\left(K_p\mathbf{e}_a + \left(K_d - C_a\right)\dot{\mathbf{e}}_a - \dot{\Phi}\dot{\mathbf{p}} + \ddot{\mathbf{a}}_d\right) \right. \\
&\phantom{=} \left.\left.\left(k_{sd}I + C\right)\Phi^{\dagger}\dot{\mathbf{a}}_d - \left(k_{sd}I + C\right)\dot{\mathbf{p}}\right)\right)
\end{split}
\end{equation}
This expression can then be simplified to
\begin{equation}
\begin{split}
\dot{V}_a &= -\epsilon\mathbf{e}_a^TK_p\mathbf{e}_a - \epsilon\mathbf{e}_a^TK_d\dot{\mathbf{e}}_a - \dot{\mathbf{e}}_a^T\left(K_d - \epsilon I\right)\dot{\mathbf{e}}_a \\
&= \mathbf{z}_a^T\begin{bmatrix} -\epsilon K_p & -\frac{\epsilon}{2}K_d \\ -\frac{\epsilon}{2}K_d & -K_d + \epsilon I\end{bmatrix} \mathbf{z}_a \\
&= \mathbf{z}_a^TM_2\mathbf{z}_a
\end{split}
\end{equation}
Therefore, for $\dot{V_a}$ to be negative, the eigenvalues of $M_2$ must be negative.  We can find the eigenvalues of $M_2$ by finding $\det\left(M_2 - \lambda I\right) = 0$:
\begin{equation}
\begin{split}
0 &= \det\left(\left(\epsilon K_p + \lambda I\right)\left(K_d  + \left(\lambda - \epsilon\right)I\right) - \frac{\epsilon^2}{4}K_d^2\right) \\
&\begin{aligned}=&\det\left(\lambda^2I + \left(\epsilon K_p + K_d - \epsilon I\right)\lambda + \epsilon K_pK_d - \epsilon^2K_p - \frac{\epsilon^2}{4}K_d^2\right) \end{aligned} \\
&\begin{aligned}=&\left(\lambda^2 + \left(\epsilon k_{p1} + k_{d1} - \epsilon\right)\lambda + \epsilon k_{p1}k_{d1} - \epsilon^2k_{p1}^2 - \frac{\epsilon^2}{4}k_{d1}^2\right) \\
& \left(\lambda^2 + \left(\epsilon k_{p2} + k_{d2} - \epsilon\right)\lambda + \epsilon k_{p2}k_{d2} - \epsilon^2k_{p2}^2 - \frac{\epsilon^2}{4}k_{d2}^2\right) \end{aligned}
\end{split}
\end{equation}
where $k_{di}$ is the $i$th diagonal element of $K_d$.  From here, we can see that to get negative real parts for the eigenvalues, the following conditions must be met:
\begin{align}
\epsilon k_{pi} + k_{di} - \epsilon &> 0 \\
\epsilon k_{pi}k_{di} - \epsilon^2k_{pi}^2 - \frac{\epsilon^2}{4}k_{di}^2 &> 0
\end{align}
for $i = \left\lbrace 1,\ 2 \right\rbrace$.

For this Lyapunov candidate, $\dot{V}_a \leq \dot{V}_a^{\prime} = \lambda_{max}(M_2)\norm{\mathbf{z}_a}$, and there are no coupling terms.

We can use the Lyapunov candidate $V_p$ from the proof of Theorem \ref{thm:LQR_SI_sys}.  If $k_1 > 0$, $k_2 > \frac{\dot{J}_{max}}{\cos\left(\theta_{max}\right)}$, and there exists some $\epsilon$ that satisfies \eqref{eq:DI_cond1}, \eqref{eq:DI_cond2}, \eqref{eq:DI_cond3}, and \eqref{eq:DI_cond4}, then by Theorem \ref{thm:GeneralStability}, there exists a Lyapunov candidate for the combined system $V = V_p + V_a$ that proves the origin of the combined system is locally asymptotically stable on a region of attraction
\begin{equation}
S = \left\lbrace\mathbf{z} \in \mathbb{R}^6\ \vert\ V\left(\mathbf{z}\right) \leq \frac{\theta_{max}^2}{2}\right\rbrace
\end{equation}
where $\mathbf{z} = \begin{bmatrix} \mathbf{z_p}^T & \mathbf{z}_a^T \end{bmatrix}^T$.
\end{proof}

\section{Proof of Theorem \ref{thm:ARISE_DI_sys}}
\label{apx:ProofARISE_DI}
\begin{proof}
We can use the same Lyapunove candidate for the swarm system from the proof for Theorem \ref{thm:LQR_DI_sys}:
\begin{equation}
V_a = \frac{1}{2}\mathbf{z}_a^T \begin{bmatrix} K_p & \epsilon I \\ \epsilon I & I \end{bmatrix} \mathbf{z}_a = \mathbf{z}_a^T M_1 \mathbf{z}_a
\end{equation}
We can show that this is positive definite and its derivative is negative definite if
\begin{align}
\epsilon k_{pi} + k_{di} - \epsilon &> 0 \\
\epsilon k_{pi}k_{di} - \epsilon^2k_{pi} - \frac{\epsilon}{4}k_{di}^2 &> 0 \\
k_{pi} + 1 &> 0 \\
k_{pi} - \epsilon^2 &> 0
\end{align}

We can use the Lyapunov candidate $V_p$ from the proof to Theorem \ref{thm:ARISE_SI_sys} for the parent system.  If  $\alpha_1 > \frac{1}{2}$, $\alpha_2 > 1$, $\beta > \zeta_{N_d} + \frac{1}{\alpha_2}\zeta_{\\dot{N}_d}$, $k_s > \frac{c_{max}}{4\eta_3}$, $\Gamma > 0$, $\Gamma^T = \Gamma$, and there exists some $\epsilon$ that satisfies \eqref{eq:DI_cond1}, \eqref{eq:DI_cond2}, \eqref{eq:DI_cond3}, and \eqref{eq:DI_cond4}, then by Theorem \ref{thm:GeneralStability}, there exists a Lyapunov candidate for the combined system $V = V_p + V_a$ that proves the origin of the combined system is locally stable where $\mathbf{z}_p, \mathbf{z}_a \to 0$ while $\mathbf{y}_p$ remains bounded on a region of attraction
\begin{equation}
S = \left\lbrace\mathbf{z} \in \mathbb{R}^9\ \vert\ \norm{\mathbf{z}}^2 \leq \frac{\eta_1}{\eta_2}\rho^2 \right\rbrace
\end{equation}
where $\mathbf{z} = \begin{bmatrix} \mathbf{z}_p^T & \mathbf{y}_p^T & \mathbf{z}_a^T \end{bmatrix}^T$ and $\eta_1$, $\eta_2$, and $\rho$ are defined in Theorem \ref{thm:ARISE_SI_sys}.
\end{proof}

\section{Proof of Theorem \ref{thm:LQR_noHomo_sys}}
\label{apx:ProofLQR_noHomo}
\begin{proof}
We break the swarm into two parts, the single integrator swarm with a state $z_{as}$, and the double integrator swarm $z_{ad}$.  We choose the Lyapunov candidate $V_a = V_{as} + V_{ad}$, where $V_{as} = \frac{1}{2}\mathbf{z}_{as}^T\mathbf{z}_{as}$ from the proof of Theorem \ref{thm:LQR_SI_sys}.  $V_{ad} = \frac{1}{2}\mathbf{z}_{ad}^TM_1\mathbf{z}_{ad}$ from the proof of Theorem \ref{thm:LQR_DI_sys}.  We can use the same proofs of these theorems to show that $V_a$ is positive definite under the given conditions.

We can show that $\dot{V}_a$ is negative definite because $\dot{V}_a = \dot{V}_{as} + \dot{V}_{sd}$, and we can show the these two terms are negative definite as in the proofs of Theorems \ref{thm:LQR_SI_sys} and \ref{thm:LQR_DI_sys}.  Also note that neither function has a coupling term, so $\dot{V}_{ac} = 0$ for the heterogeneous swarm as well.

We can use the Lyapunov candidate $V_p$ from the proof of Theorem \ref{thm:LQR_SI_sys}.  If $k_1 > 0$, $k_2 > \frac{\dot{J}_{max}}{\cos\left(\theta_{max}\right)}$, and there exists some $\epsilon$ that satisfies \eqref{eq:DI_cond1}, \eqref{eq:DI_cond2}, \eqref{eq:DI_cond3}, and \eqref{eq:DI_cond4}, then by Theorem \ref{thm:GeneralStability}, there exists a Lyapunov candidate for the combined system $V = V_p + V_a$ that proves the origin of the combined system is locally asymptotically stable on a region of attraction
\begin{equation}
S = \left\lbrace\mathbf{z} \in \mathbb{R}^6\ \vert\ V\left(\mathbf{z}\right) \leq \frac{\theta_{max}^2}{2}\right\rbrace
\end{equation}
where $\mathbf{z} = \begin{bmatrix} \mathbf{z_p}^T & \mathbf{z}_a^T \end{bmatrix}^T$.
\end{proof}

\section{Proof of Theorem \ref{thm:ARISE_noHomo_sys}}
\label{apx:ProofARISE_noHomo}
\begin{proof}
We can use the same Lyapunov candidate $V_a$ from the proof of Theorem \ref{thm:LQR_noHomo_sys} and the Lyapunov candidate $V_p$ from the proof of Theorem \ref{thm:ARISE_SI_sys}.  If  $\alpha_1 > \frac{1}{2}$, $\alpha_2 > 1$, $\beta > \zeta_{N_d} + \frac{1}{\alpha_2}\zeta_{\dot{N}_d}$, $k_s > \frac{c_{max}}{4\eta_3}$, $\Gamma > 0$, $\Gamma^T = \Gamma$, and there exists some $\epsilon$ that satisfies \eqref{eq:DI_cond1}, \eqref{eq:DI_cond2}, \eqref{eq:DI_cond3}, and \eqref{eq:DI_cond4}, then by Theorem \ref{thm:GeneralStability}, there exists a Lyapunov candidate for the combined system $V = V_p + V_a$ that proves the origin of the combined system is locally stable, where $\mathbf{z}_p, \mathbf{z}_a \to 0$ while $\mathbf{y}_p$ remains bounded on a region of attraction
\begin{equation}
S = \left\lbrace\mathbf{z} \in \mathbb{R}^9\ \vert\ \norm{\mathbf{z}}^2 \leq \frac{\eta_1}{\eta_2}\rho^2 \right\rbrace
\end{equation}
where $\mathbf{z} = \begin{bmatrix} \mathbf{z}_p^T & \mathbf{y}_p^T & \mathbf{z}_a^T \end{bmatrix}^T$ and $\eta_1$, $\eta_2$, and $\rho$ are defined in the proof of Theorem \ref{thm:ARISE_SI_sys}.
\end{proof}

\section{Supporting Lemmas}
\label{apx:lemmas}
\begin{lemma}
\label{lm:J_dot_is_bounded}
Given some $\dot{J}_{max} > 0$, the constraints on $J_{sd}$ given in theorem \ref{thm:LQR_SI_sys}, the PD control law given in \eqref{eq:ControlLawPD}, and the swarm control law given in \eqref{eq:swarmControlLaw}, we can find a bounded domain $D_{Jmax} \subset D$ where $\abs{\dot{J}_s} \leq \dot{J}_{max}$ for all $\mathbf{z} \in D_{Jmax}$ if $\pder{J_{sd}}{\tau_{sd}}$ is finite on $D$.
\end{lemma}

\begin{proof}
We start with the definition of $e_J$ from \eqref{eq:abstractErrorDef}, substitute \eqref{eq:JLimit}, and solve for $J_s$:
\begin{equation}
J_s = J_{sd} - e_J
\end{equation}
We then take the time derivative to get
\begin{equation}
\dot{J}_s = \dot{J}_{sd} - \dot{e}_J
\end{equation}
We substitute in the time derivative of \eqref{eq:ControlLawPD} and the swarm control law \eqref{eq:swarmControlLaw}, giving
\begin{equation}
\dot{J}_s = \pder{J_{sd}}{\tau_{sd}}\left(k_1\dot{\theta} + k_2\ddot{\theta}\right) - K_{s2}e_J
\end{equation}
where $k_1$ and $k_2$ are the elements of $K_{pd}$, and $K_{s2}$ is the second diagonal element of the $K$ gain matrix of the swarm controller.  Let $\abs{\pder{J_{sd}}{\tau_{sd}}} \leq \delta J_{max}$.

We now substitute the equations of motion for the parent system \eqref{eq:newDynamics}:
\begin{dmath}
\dot{J}_s = \pder{J_{sd}}{\tau_{sd}}\left(k_1\dot{\theta} + \frac{k_2}{J + J_{sd} - e_J}\left(-\cos\left(\theta\right)\tau_{sd} - \dot{J}_s\dot{\theta} - f_f\left(\dot{\theta}\right)\right)\right) - K_{s2}e_J
\end{dmath}
Now, we solve for $\dot{J}_s$ and substitute the parent system control law \eqref{eq:ControlLawPD}:
\begin{equation}
\dot{J}_s = \frac{\pder{J_{sd}}{\tau_{sd}}k_1\left(J + J_{Sd} - e_J\right)\dot{\theta} - \pder{J_{sd}}{\tau_{sd}}k_2\left(\cos\left(\theta\right)\tau_{Sd} + f_f\left(\dot{\theta}\right)\right) - K_{s2}e_J}{J + J_{Sd} - e_J + \pder{J_{sd}}{\tau_{sd}}k_2\dot{\theta}}
\end{equation}

The friction term can be bounded as follows:
\begin{equation}
\abs{f_f\left(\dot{\theta}\right)} \leq \gamma_1 + \gamma_6\abs{\dot{\theta}}
\end{equation}

We can now find $\abs{\dot{J}_s}$ and bound it with $\norm{\mathbf{z}}$:
\begin{equation}
\abs{\dot{J}_s} \leq \frac{\alpha_1\norm{\mathbf{z}}^2 + \alpha_2\norm{\mathbf{z}} + \alpha_3}{\beta_1\norm{\mathbf{z}} + \beta_2} \leq \dot{J}_{max}
\end{equation}
where
\begin{align}
\alpha_1 &= \frac{\delta J_{max}}{2}k_1 \\
\alpha_2 &= \delta J_{max}\left(k_1\left(J + J_{s,max}\right) + k_2\gamma_6\right) - K_{s2}\\
\alpha_3 &= \delta J_{max}k_2\left(\tau_{max} + \gamma_1\right) \\
\beta_1 &= \delta J_{max}k_2 \\
\beta_2 &= J + J_{s,min}
\end{align}

Thus we can conclude that there is a domain
\begin{equation}
D_{Jmax} = \left\lbrace\mathbf{z} \in D\ |\  \frac{\alpha_1\norm{\mathbf{z}}^2 + \alpha_2\norm{\mathbf{z}} + \alpha_3}{\beta_1\norm{\mathbf{z}} + \beta_2} \leq \dot{J}_{max}\right\rbrace
\end{equation}
\end{proof}

\begin{lemma}
\label{lm:NisBounded}
Given the following:
\begin{equation}
N = -Y_d\Gamma\dot{Y}_d^Tr + \dot{W} - \frac{1}{2}\dot{J}_sr + e_2 + \tau_d
\end{equation}
\begin{equation}
N_d = \dot{J}_s\ddot{\theta}_d + \left(J + J_s\right)\dddot{\theta}_d + \ddot{J}_s\dot{\theta}_d + \dot{J_s}\ddot{\theta}_d + \dot{f}_f\left(\dot{\theta}_d\right) - \dot{Y}_d\lambda + \dot{\tau}_d
\end{equation}
we can show that
\begin{equation}
\abs{\tilde{N}} \leq c_{max}\norm{\mathbf{z}}
\end{equation}
where
\begin{equation}
\mathbf{z} = \begin{bmatrix} e_1 & e_2 & r & e_\tau & e_J \end{bmatrix}
\end{equation}
\begin{equation}
\tilde{N} = N - N_d
\end{equation}
and $c_{max} > 0$.
\end{lemma}
\begin{proof}
We start by expanding $N$ and substituting in the auxilliary equation \eqref{eqn:aux_W}:
\begin{equation}
\begin{split}
N &= -Y_d\Gamma\dot{Y}_dr + \dot{J}_s\ddot{\theta}_d + \dot{J}_s\left(\alpha_1\dot{e_1} + \alpha_2 e_2\right) + \left(J + J_s\right)\dddot{\theta}_d \\
&\phantom{=} + \left(J + J_s\right)\left(\alpha_1\ddot{e}_1 + \alpha_2\dot{e}_2\right) + \ddot{J}_s\dot{\theta}_d - \ddot{J}_s\dot{e}_1 \\
&\phantom{=} + \dot{J}\ddot{\theta}_d - \dot{J}_s\ddot{e}_1 + \dot{f}_f\left(\dot{\theta}\right) - \dot{Y}_d\lambda - \frac{1}{2}\dot{J}_sr + e_2 + \dot{\tau}_d
\end{split}
\end{equation}

We then use the definitions of $e_1$, $e_2$, and $r$ to put $\tilde{N}$ in terms of these variables:
\begin{equation}
\begin{split}
\tilde{N} &= N - N_d \\
&\begin{aligned}
=&-Y_d\Gamma\dot{Y}_dr + \dot{J}_s\left(\alpha_1\dot{e_1} + \alpha_2 e_2\right) + \left(J + J_s\right)\left(\alpha_1\ddot{e}_1 - \alpha_2\dot{e}_2\right) - \ddot{J}_s\dot{e}_1 - \dot{J}_s\ddot{e}_1 \\
&+ \dot{f}_f\left(\dot{\theta}\right) - \dot{f}_f\left(\dot{\theta}_d\right) - \frac{1}{2}\dot{J}_sr + e_2
\end{aligned} \\
&\begin{aligned} 
= &\left(-\frac{1}{2}\dot{J}_s - Y_d\Gamma\dot{Y}_d\right)r - \alpha_2\left(1 + \dot{J}_s\alpha_2\right)e_2 + \left(J + J_s\right)\dot{e}_2 + \left(\dot{J}_s\alpha_1 - \ddot{J}_s\right)\dot{e}_1 \\
&+ \left(\left(J + J_s\right)\alpha_1 - \dot{J}_s\right)\ddot{e}_1 + \dot{f}_f\left(\dot{\theta}\right) - \dot{f}_f\left(\dot{\theta}_d\right)
\end{aligned} \\
&\begin{aligned}
= &\left(-\frac{1}{2}\dot{J}_s - Y_d\Gamma\dot{Y}_d\right)r + \left(1 + \dot{J}_s\alpha_2\right)e_2 - \alpha_2\left(J + J_s\right)\left(r - \alpha_2e_2\right) + \left(\dot{J}_s\alpha_1 - \ddot{J}_s\right)\left(e_2 - \alpha_1e_1\right) \\
&+ \left(\left(J + J_s\right)\alpha_1 - \dot{J}_s\right)\left(r - \alpha_1\left(e_2 - \alpha_1e_1\right) - \alpha_2e_2\right) + \dot{f}_f\left(\dot{\theta}\right) - \dot{f}_f\left(\dot{\theta}_d\right)
\end{aligned} \\
&\begin{aligned}
= &\left(-\frac{1}{2}\dot{J}_s - Y_d\Gamma\dot{Y}_d + \left(J + J_s\right)\left(\alpha_1 - \alpha_2\right) - \dot{J}_s\right)r + \left(1 + \left(2\dot{J}_s + \left(J + J_s\right)\left(\alpha_2 - \alpha_1\right)\right)\alpha_2 + J_s\alpha_1 \right.\\
&- \left.\ddot{J}_s - \left(J + J_s\right)\alpha_1^2 + \dot{J}_s\alpha_1\right)e_2 + \left(\ddot{J}_s\alpha_1 - 2\dot{J}_s\alpha_1^2 + \left(J + J_s\right)\alpha_1^3\right)e_1 + \dot{f}_f\left(\dot{\theta}\right) - \dot{f}_f\left(\dot{\theta}_d\right)
\end{aligned}
\end{split}
\end{equation}

We can find limits on $J_s$ and its derivatives based on our chosen limits on $\tau_{sd}$ and its derivatives based on the relation given in \eqref{eq:JLimit}:
\begin{align}
\abs{\dot{J}_s} \leq \dot{J}_{max} &= \abs{\pder{J_{sd}}{\tau_d}\dot{\tau}_{max}} \\
\abs{\ddot{J}_s} \leq \ddot{J}_{max} &= \abs{\pder{J_{sd}}{\tau_d}\ddot{\tau}_{max} + \pdder{J_{sd}}{\tau_d}\dot{\tau}_{max}^2}
\end{align} 
where $\tau_{max}$, $\dot{\tau}_{max}$, and $\ddot{\tau}_{max}$ are the limits on $\tau_{sd}$ and its first and second derivative, respectively.  With these limits, we can say that
\begin{equation}
\begin{split}
\abs{\tilde{N}} &\leq \left(\frac{1}{2}\dot{J}_{max} + \abs{Y_d\Gamma\dot{Y}_d} + \left(J + J_{max}\right)\abs{\alpha_1 - \alpha_2} + \dot{J}_{max}\right)\abs{r} \\
&\phantom{=} + \abs{1 + \left(2\dot{J}_{max} + \left(J + J_{max}\right)\left(\alpha_2 - \alpha_1\right)\right)\alpha_2 + J_{max}\alpha_1}\abs{e_2} \\
&\phantom{=} + \abs{\ddot{J}_{max} + J\alpha_1^2 + \dot{J}_{max}\alpha_1}\abs{e_2} +
\left(\ddot{J}_{max}\alpha_1 + 2\dot{J}_{max}\alpha_1^2 + \left(J + \dot{J}_{max}\right)\alpha_1^3\right)\abs{e_1} \\
&\phantom{=} + \abs{\dot{f}_f\left(\dot{\theta}\right) - \dot{f}_f\left(\dot{\theta}_d\right)}
\end{split}
\end{equation}

We now consider specifically the term $\dot{f}_f\left(\dot{\theta}\right) - \dot{f}_f\left(\dot{\theta}_d\right)$.  For the parameters that are not adaptive, we use $\bar{\gamma}_i$ to represent the best \textit{a priori} estimate of the value of the parameter; then \eqref{eq:frictionModel} gives
\begin{equation}
\begin{split}
\bar{f}_f\left(\dot{\theta}\right) - f_f\left(\dot{\theta}_d\right) &= \gamma_1\left(\tanh\left(\gamma_2\dot{\theta}\right) - \tanh\left(\gamma_3\dot{\theta}\right)\right) + \gamma_4\tanh\left(\gamma_5\dot{\theta}\right) + \gamma_6\dot{\theta} \\
&\phantom{=} - \gamma_1\left(\tanh\left(\bar{\gamma}_2\dot{\theta}_d\right) - \tanh\left(\bar{\gamma}_3\dot{\theta}_d\right)\right) - \gamma_4\tanh\left(\bar{\gamma}_5\dot{\theta}_d\right) - \gamma_6\dot{\theta}_d
\end{split}
\end{equation}
Now we take the time derivative and get
\begin{equation}
\begin{split}
\dot{\bar{f}}_f\left(\dot{\theta}\right) - \dot{f}_f\left(\dot{\theta}_d\right) &= \ddot{\theta}\left(\gamma_1\left(\gamma_2\sech^2\left(\gamma_2\dot{\theta}\right) - \gamma_3\sech^2\left(\gamma_3\dot{\theta}\right)\right) +  \gamma_4\gamma_5\sech^2\left(\gamma_5\dot{\theta}\right) + \gamma_6\right) \\
&\phantom{=} - \ddot{\theta}_d\left(\gamma_1\left(\bar{\gamma}_2\sech^2\left(\bar{\gamma}_2\dot{\theta}_d\right) - \bar{\gamma}_3\sech^2\left(\bar{\gamma}_3\dot{\theta}_d\right)\right) + \gamma_4\bar{\gamma}_5\sech^2\left(\bar{\gamma}_5\dot{\theta}_d\right) +\gamma_6\right)
\end{split}
\end{equation}
We note that
\begin{equation}
0 \leq \sech\left(x\right) \leq 1
\end{equation}
and that $\gamma_2 > \gamma_3$ for this to be a positive function and thus work as a friction model.  Based on this knowledge, we can say
\begin{equation}
\begin{split}
\abs{\dot{\bar{f}}_f\left(\dot{\theta}\right) - \dot{f}_f\left(\dot{\theta}_d\right)} &\leq \ddot{\theta} \left(\gamma_1\gamma_2 - \gamma_1\gamma_3 + \gamma_4\gamma_5 + \gamma_6\right) - \ddot{\theta}_d\left(\gamma_1\bar{\gamma}_2 - \gamma_1\bar{\gamma}_3 + \gamma_4\bar{\gamma}_5 + \gamma_6\right) \\
&\leq \left(\gamma_1\hat{\gamma}_2 - \gamma_1\hat{\gamma}_3 +  \gamma_4\hat{\gamma}_5 + \gamma_6\right) \left(r - \alpha_1\left(e_2 - \alpha_1e_1\right) - \alpha_2e_2\right) \\
&\leq c \left(r - \alpha_1\left(e_2 - \alpha_1e_1\right) - \alpha_2e_2\right)
\end{split}
\end{equation}
where $\hat{\gamma}_2 = max\left\lbrace\gamma_2, \bar{\gamma}_2\right\rbrace$, $\hat{\gamma}_3 = min\left\lbrace\gamma_3, \bar{\gamma}_3\right\rbrace$, and $\hat{\gamma}_5 = max\left\lbrace\gamma_5, \bar{\gamma}_5\right\rbrace$.

Therefore, we can now say that
\begin{equation}
\abs{\tilde{N}} \leq c_1r + c_2e_2 + c_3e_1 \leq c_{max}\norm{\mathbf{z}}
\end{equation}
where 
\begin{equation}
c_1 = \left(\frac{1}{2}\dot{J}_{max} + \abs{Y_d\Gamma\dot{Y}_d} + \left(J + J_{max}\right)\abs{\alpha_1 - \alpha_2} + \dot{J}_{max} + c\right)
\end{equation}
\begin{equation}
\begin{split}
c_2 &= \abs{1 + \left(2\dot{J}_{max} + \left(J + J_{max}\right)\left(\alpha_2 - \alpha_1\right)\right)\alpha_2} \\
&\phantom{=} + \abs{J_{max}\alpha_1 + \ddot{J}_{max} + J\alpha_1^2 + \dot{J}_{max}\alpha_1 - \left(\alpha_1 + \alpha_2\right)c}
\end{split}
\end{equation}
\begin{equation}
c_3 = \left(\ddot{J}_{max}\alpha_1 + 2\dot{J}_{max}\alpha_1^2 + \left(J + \dot{J}_{max}\right)\alpha_1^3 + \alpha_1^2c\right)
\end{equation}
and $c_{max} = max\left\lbrace c_1, c_2, c_3 \right\rbrace$.  Note that the $Y_d\Gamma\dot{Y}_d$ term is bounded as it is a bounded function of the desired trajectory $\theta_d$ and its first three derivatives, which are bounded.
\end{proof}

\begin{lemma}
\label{lm:EisBounded}
Given the following:
\begin{equation}
E = \dot{\theta}\sin\left(\theta\right)e_\tau - \cos\left(\theta\right)\dot{e}_\tau
\end{equation}
\begin{equation}
E_d = \dot{\theta}_d e_\tau
\end{equation}
the following condition holds:
\begin{equation}
\abs{\tilde{E}} \leq \rho_E\left(\norm{\mathbf{z}}\right)\norm{\mathbf{z}}
\end{equation}
where
\begin{equation}
\mathbf{z} = \begin{bmatrix} e_1 & e_2 & r & e_\tau & e_J \end{bmatrix}
\end{equation}
\begin{equation}
\tilde{E} = E - E_d
\end{equation}
\end{lemma}
\begin{proof}
We begin by substituting $E$ and $E_d$ into the definition of $\tilde{E}$ and applying the control law \eqref{eq:swarmControlLaw}:
\begin{equation}
\begin{split}
\tilde{E} &= E - E_d \\ 
&= \dot{\theta}\sin\left(\theta\right)e_\tau - \cos\left(\theta\right)\dot{e}_\tau - \dot{\theta}_de_\tau \\
&= \dot{\theta}\sin\left(\theta\right)e_\tau - \cos\left(\theta\right)\left(\dot{\tau}_{sd} - k_1e_\tau - \dot{\tau}_sd\right) - \dot{\theta}_de_\tau \\
&=\dot{\theta}\sin\left(\theta\right)e_\tau + \cos\left(\theta\right)k_1e_\tau - \dot{\theta}_de_\tau
\end{split}
\end{equation}
We can then show that
\begin{equation}
\begin{split}
\tilde{E} &\leq \dot{\theta}e_\tau + k_1e_\tau - \dot{\theta}_de_\tau \\
&\leq \dot{\theta}_de_\tau - \dot{e}_1e_\tau + k_1e_\tau - \dot{\theta}_de_\tau \\
&\leq \left(k_1 -e_2 + \alpha_1e_1\right)e_\tau
\end{split}
\end{equation}
\begin{equation}
\begin{split}
\abs{\tilde{E}} &\leq \left(k_1 + \left(1 + \alpha_1\right)\norm{\mathbf{z}}\right)\norm{\mathbf{z}} \\
&\leq \rho_E\left(\norm{\mathbf{z}}\right)\norm{\mathbf{z}}
\end{split}
\end{equation}
where $k_1$ is the first term on the diagonal of $K$.
\end{proof}

\begin{lemma}
\label{lm:PIsPosDef}
Given the function
\begin{equation}
L\left(t\right) = r\left(N_d\left(t\right) - \beta sgn\left(e_2\left(t\right)\right)\right)
\end{equation}
If $\beta > \zeta_{N_d} + \frac{1}{\alpha_2}\zeta_{\dot{N}_d}$, then
\begin{equation}
\displaystyle\int_{t_0}^tL\left(\tau\right)d\tau \leq \beta\norm{e_2\left(t_0\right)} - e_2\left(t_0\right)N_d\left(t_0\right)
\label{eq:Lemma1Inequality}
\end{equation}
\end{lemma}
\begin{proof}
\begin{equation}
\begin{split}
\displaystyle\int_{t_0}^tL\left(\tau\right)d\tau &= \displaystyle\int_{t_0}^tr\left(N_d\left(\tau\right) - \beta sgn\left(e_2\left(\tau\right)\right)\right)d\tau \\
&= \displaystyle\int_{t_0}^t\dot{e}_2N_d + \alpha_2e_2N_d - \dot{e}_2\beta sgn\left(e_2\right) - \alpha_2e_2\beta sgn\left(e_2\right)d\tau \\
&= \displaystyle\int_{t_0}^t\alpha_2e_2\left(N_d + \beta sgn\left(e_2\right)\right)d\tau + \displaystyle\int_{t_0}^t\dot{e}_2N_dd\tau - \displaystyle\int_{t_0}^t\dot{e}_2\beta sgn\left(e_2\right)d\tau
\end{split}
\end{equation}
Using integration by parts,
\begin{equation}
\begin{split}
\displaystyle\int_{t_0}^tL\left(\tau\right)d\tau &= \displaystyle\int_{t_0}^t\alpha_2e_2\left(N_d + \beta sgn\left(e_2\right)\right)d\tau + e_2N_d\rvert_{t_0}^t - \displaystyle\int_{t_0}^te_2\dot{N}_dd\tau - \displaystyle\int_{t_0}^t\dot{e}_2\beta sgn\left(e_2\right)d\tau \\
&= \displaystyle\int_{t_0}^t\alpha_2e_2\left(N_d + \frac{1}{\alpha_2}\dot{N}_d - \beta sgn\left(e_2\right)\right)d\tau + e_2\left(t\right)N_d\left(t\right) - e_2\left(t_0\right)N_d\left(t_0\right) - \beta\abs{e_2\left(t\right)} + \beta\abs{e_2\left(t_0\right)} \\
&\leq \displaystyle\int_{t_0}^t\alpha_2\abs{e_2}\left(\abs{N_d} + \frac{1}{\alpha_2}\abs{\dot{N}_d} - \beta\right)d\tau + \left(\abs{N_d\left(t\right)} - \beta\right)\abs{e_2\left(t\right)} + \beta\abs{e_2\left(t_0\right)} - e_2\left(t_0\right)N_d\left(t_0\right)
\end{split}
\end{equation}

From here we can see that if $\beta$ satisfies the given conditions, \eqref{eq:Lemma1Inequality} holds.  This Lemma and proof are adopted from \cite{Xian2004} and \cite{Patre2006}.
\end{proof}


\ifCLASSOPTIONcaptionsoff
  \newpage
\fi

\bibliographystyle{IEEEtran}
\bibliography{ControllingAParentSystem_citations}

\begin{thebibliography}{10}
\providecommand{\url}[1]{#1}
\csname url@samestyle\endcsname
\providecommand{\newblock}{\relax}
\providecommand{\bibinfo}[2]{#2}
\providecommand{\BIBentrySTDinterwordspacing}{\spaceskip=0pt\relax}
\providecommand{\BIBentryALTinterwordstretchfactor}{4}
\providecommand{\BIBentryALTinterwordspacing}{\spaceskip=\fontdimen2\font plus
\BIBentryALTinterwordstretchfactor\fontdimen3\font minus
  \fontdimen4\font\relax}
\providecommand{\BIBforeignlanguage}[2]{{%
\expandafter\ifx\csname l@#1\endcsname\relax
\typeout{** WARNING: IEEEtran.bst: No hyphenation pattern has been}%
\typeout{** loaded for the language `#1'. Using the pattern for}%
\typeout{** the default language instead.}%
\else
\language=\csname l@#1\endcsname
\fi
#2}}
\providecommand{\BIBdecl}{\relax}
\BIBdecl

\bibitem{tan2013}
Y.~Tan and Z.-y. Zheng, ``Research advance in swarm robotics,'' \emph{Defence
  Technology}, vol.~9, no.~1, pp. 18--39, 2013.

\bibitem{Habibi2015}
G.~Habibi, Z.~Kingston, W.~Xie, M.~Jellins, and J.~McLurkin, ``Distributed
  centroid estimation and motion controllers for collective transport by
  multi-robot systems,'' in \emph{Robotics and Automation (ICRA), 2015 IEEE
  International Conference on}.\hskip 1em plus 0.5em minus 0.4em\relax IEEE,
  2015, pp. 1282--1288.

\bibitem{Pereira2004}
G.~A. Pereira, M.~F. Campos, and V.~Kumar, ``Decentralized algorithms for
  multi-robot manipulation via caging,'' \emph{The International Journal of
  Robotics Research}, vol.~23, no. 7-8, pp. 783--795, 2004.

\bibitem{Fink2008}
J.~Fink, M.~A. Hsieh, and V.~Kumar, ``Multi-robot manipulation via caging in
  environments with obstacles,'' in \emph{Robotics and Automation, 2008. ICRA
  2008. IEEE International Conference on}.\hskip 1em plus 0.5em minus
  0.4em\relax IEEE, 2008, pp. 1471--1476.

\bibitem{Gross2009}
R.~Gross and M.~Dorigo, ``Towards group transport by swarms of robots,''
  \emph{International Journal of Bio-Inspired Computation}, vol.~1, no. 1-2,
  pp. 1--13, 2009.

\bibitem{Goodarzi2015}
F.~A. Goodarzi and T.~Lee, ``Dynamics and control of quadrotor uavs
  transporting a rigid body connected via flexible cables,'' in \emph{American
  Control Conference (ACC), 2015}.\hskip 1em plus 0.5em minus 0.4em\relax IEEE,
  2015, pp. 4677--4682.

\bibitem{Klausen2015}
K.~Klausen, T.~I. Fossen, T.~A. Johansen, and A.~P. Aguiar, ``Cooperative
  path-following for multirotor uavs with a suspended payload,'' in
  \emph{Control Applications (CCA), 2015 IEEE Conference on}.\hskip 1em plus
  0.5em minus 0.4em\relax IEEE, 2015, pp. 1354--1360.

\bibitem{Hamed2017}
K.~A. Hamed and R.~D. Gregg, ``Decentralized feedback controllers for robust
  stabilization of periodic orbits of hybrid systems: Application to bipedal
  walking,'' \emph{IEEE Transactions on Control Systems Technology}, vol.~25,
  no.~4, pp. 1153--1167, 2017.

\bibitem{Blower2012}
C.~J. Blower, W.~Lee, and A.~M. Wickenheiser, ``The development of a
  closed-loop flight controller with panel method integration for gust
  alleviation using biomimetic feathers on aircraft wings,'' in \emph{SPIE
  Smart Structures and Materials+ Nondestructive Evaluation and Health
  Monitoring}.\hskip 1em plus 0.5em minus 0.4em\relax International Society for
  Optics and Photonics, 2012, pp. 83\,390I--83\,390I.

\bibitem{Boberg2015}
M.~Boberg, G.~Feltrin, and A.~Martinoli, ``A novel bridge section model endowed
  with actively controlled flap arrays mitigating wind impact,'' in
  \emph{Robotics and Automation (ICRA), 2015 IEEE International Conference
  on}.\hskip 1em plus 0.5em minus 0.4em\relax IEEE, 2015, pp. 1837--1842.

\bibitem{Belta2004}
C.~Belta and V.~Kumar, ``Abstraction and control for groups of robots,''
  \emph{Robotics, IEEE Transactions on}, vol.~20, no.~5, pp. 865--875, 2004.

\bibitem{Belta2005}
C.~Belta, V.~Isler, and G.~J. Pappas, ``Discrete abstractions for robot motion
  planning and control in polygonal environments,'' \emph{IEEE Transactions on
  Robotics}, vol.~21, no.~5, pp. 864--874, 2005.

\bibitem{Belta2007}
C.~Belta, A.~Bicchi, M.~Egerstedt, E.~Frazzoli, E.~Klavins, and G.~J. Pappas,
  ``Symbolic planning and control of robot motion [grand challenges of
  robotics],'' \emph{IEEE Robotics \& Automation Magazine}, vol.~14, no.~1, pp.
  61--70, 2007.

\bibitem{Michael2006}
N.~Michael, C.~Belta, and V.~Kumar, ``Controlling three dimensional swarms of
  robots.'' in \emph{ICRA}.\hskip 1em plus 0.5em minus 0.4em\relax Citeseer,
  2006, pp. 964--969.

\bibitem{Michael2008}
N.~Michael and V.~Kumar, ``Controlling shapes of ensembles of robots of finite
  size with nonholonomic constraints,'' \emph{Proceedings of Robotics: Science
  and Systems IV}, 2008.

\bibitem{Michael2008Aerial}
N.~Michael, J.~Fink, and V.~Kumar, ``Controlling ensembles of robots via a
  supervisory aerial robot,'' \emph{Advanced Robotics}, vol.~22, no.~12, pp.
  1361--1377, 2008.

\bibitem{Michael2010}
N.~Michael, J.~Fink, S.~Loizou, and V.~Kumar, ``Architecture, abstractions, and
  algorithms for controlling large teams of robots: Experimental testbed and
  results,'' in \emph{Robotics Research}.\hskip 1em plus 0.5em minus
  0.4em\relax Springer, 2010, pp. 409--419.

\bibitem{Egerstedt2003}
M.~Egerstedt, ``Motion description languages for multi-modal control in
  robotics,'' in \emph{Control Problems in Robotics}.\hskip 1em plus 0.5em
  minus 0.4em\relax Springer, 2003, pp. 75--89.

\bibitem{Crandall2016}
K.~L. Crandall, C.~Whitehead, S.~Dong, and A.~Wickenheiser, ``Using abstraction
  for swarm control of a parent system,'' in \emph{2016 IEEE International
  Conference on Robotics and Automation (ICRA)}.\hskip 1em plus 0.5em minus
  0.4em\relax IEEE, 2016, pp. 5344--5349.

\bibitem{Khalil2002}
H.~K. Khalil, \emph{Nonlinear Systems}, 3rd~ed.\hskip 1em plus 0.5em minus
  0.4em\relax New Jersey: Prentice Hall, Inc., 2002.

\bibitem{Xian2004}
B.~Xian, D.~M. Dawson, M.~S. de~Queiroz, and J.~Chen, ``A continuous asymptotic
  tracking control strategy for uncertain nonlinear systems,'' \emph{IEEE
  Transactions on Automatic Control}, vol.~49, no.~7, pp. 1206--1211, 2004.

\bibitem{Patre2006}
P.~M. Patre, W.~MacKunis, C.~Makkar, and W.~E. Dixon, ``Asymptotic tracking for
  systems with structured and unstructured uncertainties,'' in \emph{Decision
  and Control, 2006 45th IEEE Conference on}.\hskip 1em plus 0.5em minus
  0.4em\relax IEEE, 2006, pp. 441--446.

\bibitem{Patre2008}
P.~M. Patre, W.~MacKunis, K.~Kaiser, and W.~E. Dixon, ``Asymptotic tracking for
  uncertain dynamic systems via a multilayer neural network feedforward and
  rise feedback control structure,'' \emph{Automatic Control, IEEE Transactions
  on}, vol.~53, no.~9, pp. 2180--2185, 2008.

\bibitem{Slotine1991}
J.-J.~E. Slotine, W.~Li \emph{et~al.}, \emph{Applied nonlinear control}.\hskip
  1em plus 0.5em minus 0.4em\relax Prentice hall Englewood Cliffs, NJ, 1991,
  vol. 199, no.~1, ch. 4.5, pp. 122--126.

\bibitem{Franklin2010}
G.~F. Franklin, J.~D. Powell, and A.~Emami-Naeini, \emph{Feedback Control of
  Dynamic Systems}, 6th~ed.\hskip 1em plus 0.5em minus 0.4em\relax Pearson,
  2010, ch.~3, pp. 132--134.

\end{thebibliography}

%

\begin{IEEEbiography}
[{\includegraphics[width=1in,height=1.25in, clip,keepaspectratio]{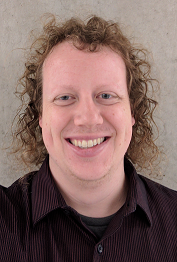}}]{Kyle L. Crandall}
(M\textquotesingle 14) received the B.S. and M.S. degrees in mechanical engineering, with a focus on robotics, in 2015 from the University of Utah, Salt Lake City.  He is currently a Ph.D. candidate at The George Washington University, Washington, DC.  His research interests are in control of multi-agent systems.  He is a student member of the IEEE.
\end{IEEEbiography}

\begin{IEEEbiography}[{\includegraphics[width=1in,height=1.25in,clip,keepaspectratio]{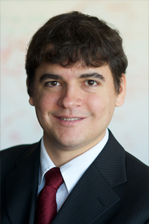}}]{Adam M. Wickenheiser}
(M\textquotesingle 08) received the B.S. degree in mechanical engineering (with a minor in applied mathematics) in 2002 and the M.S. and Ph.D. degrees in aerospace engineering from Cornell University in 2006 and 2008, respectively.

Since 2018, he has been on the Faculty of the Department of Mechanical Engineering at the University of Delaware, where he is currently an Associate Professor.  From 2010-2018, he was an Assistant Professor in the Department of Mechanical \& Aerospace Engineering at the George Washington University. From 2008-2009, he was a postdoctoral associate with the Sibley School of Mechanical \& Aerospace Engineering at Cornell University. His current research interests include bio-inspired flight, multi-functional materials and systems, and energy harvesting for autonomous systems.

Prof. Wickenheiser has served as the Chair of the Energy Harvesting Technical Committee of the American Society of Mechanical Engineers (ASME) from 2014-2016, and is presently a member of the International Organizing Committee for the International Conference on Adaptive Structures and Technologies (ICAST). He was the recipient of the 2011 Intelligence Community Young Investigator Award.
\end{IEEEbiography}



\enlargethispage{-5in}

\end{document}